%% file: CVaR_TS_arXiv.tex
\documentclass{article}
\input{raw/Template}

\newtheorem{theorem}{Theorem}
\newcommand{\vt}[1]{{\color{red}#1}}
\newcommand{\jc}[1]{{\color{blue}#1}}
\renewcommand{\vt}[1]{{#1}}
\renewcommand{\jc}[1]{{#1}}
\usepackage{microtype}
\usepackage{graphicx}
\usepackage{subfigure}
\usepackage{booktabs} 

\usepackage{hyperref}
\usepackage{enumitem}%
\setlist[itemize]{noitemsep, topsep=0pt}
\setlist[enumerate]{noitemsep, topsep=0pt}

\allowdisplaybreaks[2]

\usepackage[accepted]{icml2020}

\icmltitlerunning{Risk-Constrained Thompson Sampling for Gaussian  CVaR Bandits}

\begin{document}

\twocolumn[
\icmltitle{Risk-Constrained Thompson Sampling for  Gaussian CVaR Bandits}



\icmlsetsymbol{equal}{*}

\begin{icmlauthorlist}
\icmlauthor{Joel Q. L. Chang}{math}
\icmlauthor{Qiuyu Zhu}{iora}
\icmlauthor{Vincent Y. F. Tan}{math,iora,ece}
\end{icmlauthorlist}

\icmlaffiliation{math}{Department of Mathematics, National University of Singapore, Singapore}
\icmlaffiliation{iora}{Institute of Operations Research and Analytics, National University of Singapore, Singapore}
\icmlaffiliation{ece}{Department of Electrical and Computer Engineering, National University of Singapore, Singapore}

\icmlcorrespondingauthor{Joel Q. L. Chang}{joel.chang@u.nus.edu}
\icmlcorrespondingauthor{Qiuyu Zhu}{qiuyu\_zhu@u.nus.edu}
\icmlcorrespondingauthor{Vincent Y. F. Tan}{vtan@nus.edu.sg}

\icmlkeywords{Multi-armed bandits}

\vskip 0.3in
]



\printAffiliationsAndNotice{}  

\begin{abstract}
The multi-armed bandit (MAB) problem is a ubiquitous decision-making problem that exemplifies the exploration-exploitation tradeoff. Standard formulations exclude risk in decision making. Risk notably complicates the basic reward-maximising objective, in part because  there is no universally agreed definition of it. In this paper, we consider  a popular risk measure in quantitative finance known as the  Conditional Value at Risk (CVaR). We explore the performance of a Thompson Sampling-based algorithm CVaR-TS under this risk  measure. We provide comprehensive comparisons between our regret bounds with state-of-the-art L/UCB-based algorithms in comparable settings and  demonstrate their clear improvement in performance. \vt{We show that the regret bounds of  CVaR-TS approach   \jc{instance-dependent} lower bounds for the same problem in some parameter regimes.}
We also \vt{perform} numerical simulations to empirically verify that CVaR-TS \vt{significantly}  outperforms other L/UCB-based algorithms \vt{in terms of regret and also in flagging an infeasible instance as feasible and vice versa}. 
\end{abstract}

\section{Introduction}
\label{submission}

The multi-armed bandit (MAB) problem analyses sequential decision making, where the learner has access to partial feedback. This problem is   applicable to a variety of real-world applications, such as clinical trials, online advertisement, network routing, and resource allocation. In the well-known stochastic MAB setting, a player chooses among $K$ arms, each characterised by an independent reward distribution. During each period, the player plays one arm and observes a random reward from that arm, incorporates the information in choosing the next arm to select. The player repeats the process for a horizon containing $n$ periods. In each period, the player faces a dilemma whether to  explore the potential value of other arms or to exploit the arm that the player believes offers the highest estimated reward.

In each of the above-mentioned applications, risk is not taken into account. In this regard, the MAB problem should be  explored in a more sophisticated setting, where the player wants to maximise one's reward while minimising one's risk incurred, subject to a ``maximum risk'' condition. This paper proposes the first Thompson sampling-based learning algorithm which minimises the \emph{Conditional Value at Risk (CVaR)} risk measure \citep{rockafellar2000optimization}, also known as the \emph{expected shortfall (ES)}.
\subsection{Related Work}
\label{related}
Various analyses of MABs involving risk measures have been undertaken.
\citet{sani2013riskaversion} considered the \emph{mean-variance} as their risk measure. Each arm  $i$ was distributed according to a Gaussian with mean $\mu_i \in [0,1]$ and variance $\sigma_i^2 \in [0,1]$. The authors provided an LCB-based algorithm with accompanying regret analyses.
\citet{galichet2013exploration} proposed  the Multi-Armed Risk-Aware Bandit (\textsc{MaRaB}) algorithm with the objective of minimising the number of pulls of risky arms, using CVaR as their risk measure.
\citet{Vakili_2016} showed that the instance-dependent   and   instance-independent regrets in terms of the mean-variance of the reward process over a horizon $n$ are lower bounded by $\Omega(\log  n)$ and $\Omega(n^{2/3})$ respectively.
\citet{sun2016riskaware} considered contextual bandits with risk constraints, and developed a meta algorithm utilizing the online mirror descent algorithm which achieves near-optimal regret in terms of minimizing the total cost.
\citet{huo2017riskaware} studied applications of risk-aware MAB into portfolio selection, achieving a balance between risk and return.
\citet{kyn2020cvar} similarly used CVaR as their risk measure and proposed the CVaR-UCB algorithm, which chooses the arm with the highest gap with respect to the arm with the highest CVaR. Note that the authors associated the highest CVaR to the highest reward, while our paper considers the highest CVaR to result in the highest loss. Thus, the discussions in \citet{kyn2020cvar} apply analogously.

The paper closest to our work is that by \citet{kagrecha2020constrained}, who regard a large CVaR as  corresponding to a large loss incurred by the user. They proposed and analysed RC-LCB and considered the scenario where users have a predetermined risk tolerance level $\tau$. Any arm whose CVaR exceeds $\tau$ is deemed infeasible. This motivates the authors to define three different regrets, each corresponding to the number of times the non-optimal arms of different classes (e.g., feasible/infeasible) were pulled. This problem formulation also includes that of \citet{kyn2020cvar} in the ``infeasible instance'' case. Having drawn inspiration from \citet{kagrecha2020constrained} in terms of the problem setup, in Section~\ref{problem}, we provide definitions of these classes of arms later which were originally defined in their paper. 
\citet{zhu2020thompson} provided a Thompson sampling-based algorithm for the scenario in \citet{sani2013riskaversion}, where the arm distributions are Gaussian or Bernoulli. Our paper seeks to explore the efficacy of Thompson sampling in the problem setting proposed by~\citet{kagrecha2020constrained}  and to demonstrate, theoretically and empirically, the improvement of Thompson sampling over confidence bound-based approaches. Finally, we note the independent  and contemporaneous work by \citet{baudry2020thompson} on designing and analyzing Thompson sampling-based algorithms $\alpha$-NPTS for {\em bounded} rewards and $\alpha$-Multinomial-TS for {\em discrete multinomial distributions}; we consider a different class of reward distributions (Gaussian whose support is unbounded) here and also {\em three} different regret criteria (Definition~\ref{def:regrets}).
\subsection{Contributions}
\label{contributions}
\begin{itemize}
	\item \textbf{CVaR-TS Algorithm:} We design  CVaR-TS, an algorithm that is similar to the structure of RC-LCB in \citet{kagrecha2020constrained} but using   Thompson sampling \citep{thompson1933likelihood,agrawal2012analysis} as explored for mean-variance bandits in \citet{zhu2020thompson}. 
	\item \vt{\textbf{Comprehensive regret analyses:} We provide theoretical analyses of the algorithms and show that in a  variety of parameter regimes, the bounds on three   types of regrets outperform those of \citet{kagrecha2020constrained}, \citet{bhat2019concentration}, \citet{kyn2020cvar} and \citet{tamkin2020dist}. 
	In the risk-neutral setting, they coincide with existing  bounds under other risk measures,  such as the mean-variance \citep{zhu2020thompson}. We argue that our analytical techniques apply more broadly to other risk measures such as the {\em entropic risk}. Finally, we judiciously particularize the instance-dependent lower bounds derived by  \citet{kagrecha2020constrained} to Gaussian CVaR MABs and show that CVaR-TS meets these lower bounds in some regimes. }
	\item \textbf{Numerical simulations:} We provide an extensive set of simulations to demonstrate that under certain regimes, our algorithm based on Thompson sampling~\citep{thompson1933likelihood} consistently outperforms RC-LCB \citep{kagrecha2020constrained} and \textsc{MaRaB}  \citep{galichet2013exploration}. Furthermore, in \citet{kagrecha2020constrained}, due to the nature of the constants that were not explicitly defined, numerical simulations could not be implemented. By estimating the constants that they used based on explicit concentration bounds  in \citet{a2019concentration} and \citet{fournier2015rate}, we implemented RC-LCB and demonstrated significant improvements in performance using Thompson sampling. \vt{Finally, we show empirically that the empirical probability of CVaR-TS  declaring  that an instance that is infeasible is flagged as feasible and vice versa is lower than competing methods. }
\end{itemize}
This paper is structured as follows. We introduce the formulation of   CVaR MAB   in Section~\ref{problem}. In Section~\ref{algorithms}, we present   CVaR-TS   and demonstrate that it is a Thompson sampling analogue of \citet{kagrecha2020constrained} and \citet{kyn2020cvar}. In Section~\ref{regret_analyses} and \ref{proof_outline},   we state upper bounds on the regret and outline their proofs respectively. In Section~\ref{numerical_simulations}, we provide numerical simulations to validate the regret bounds. We conclude our discussions in Section~\ref{conclusion}, suggesting avenues for future research. \vt{We  defer detailed proofs of the theorems  to the supplementary material.}

\section{Problem formulation}
\label{problem}

In this section we  define the CVaR MAB problem. Throughout the paper, denote $[m] = \sett{1,\dots,m}$ for any $m \in \NN$ and ${(x)}^+ = \max \sett{0,x}$ for $x \in \RR$.
\begin{definition}
\label{def: CVaR}
For any random variable $X$, given a confidence level $\alpha \in [0,1)$, we define the \vocab{Value at Risk (VaR)} and \vocab{Conditional Value at Risk (CVaR)} metrics of $X$ by
	\begin{align}
		\VaR{\alpha}{X} &= {\inf \sett{v\in \RR : \PP(X \leq v) \geq \alpha}},\ \text{and} \label{eqn: VaR definition}\\
		\CVaR{\alpha}{X} &= {\VaR{\alpha}{X} + \frac{1}{1-\alpha} \EE\big[\nonneg{(X - \VaR{\alpha}{X})}}\big]. \label{eqn: CVaR definition}
	\end{align}	
\end{definition}
As explained by \citet{kagrecha2020constrained}, $\VaR{\alpha}{X}$ is the worst case loss corresponding to a confidence level $\alpha$,  which is usually taken to be in the set $[0.90,1)$  \citep{rockafellar2000optimization}, where $X$ is the loss associated with a portfolio. Working with a continuous cumulative distribution function (CDF) $F_X(\cdot)$ of $X$ that is strictly increasing over its support, direct computations give $\CVaR{\alpha}{X} = \EE[X | X \geq \VaR{\alpha}{X}]$, and thus $\CVaR{\alpha}{X}$ can be interpreted as the expected loss given that the loss exceeds $\VaR{\alpha}{X}$.
We remark that CVaR is usually preferred to VaR as a risk measure  since it is coherent \citep{Artzner} and satisfies more mathematically useful properties for analysis. In our paper, we will be working with Gaussian random variables  $X \Fol \NDist(\mu,\sigma^2)$ whose CDF satisfies the continuity and strict monotonicity assumptions  of $F_X(\cdot)$. Letting $\Phi$ denote the CDF of $Z \Fol \NDist(0,1)$, direct computations using~(\ref{eqn: VaR definition}) and (\ref{eqn: CVaR definition}) yield 
	\begin{align}
\CVaR{\alpha}{Z} &= {\frac{1}{(1-\alpha)\sqrt{2\pi}}\exp  \Big(-\frac{1}{2}\paren{\inv{\Phi}(\alpha)}^2} \Big), \label{eqn: CVaR standard Gaussian}\\*
	\CVaR{\alpha}{X} &= {\mu \Big(\frac{\alpha}{1-\alpha}\Big)+ \sigma \CVaR{\alpha}{Z}}. \label{eqn: CVaR Gaussian}
	\end{align} 
	In the rest of the paper, we denote $c_\alpha^* = \CVaR{\alpha}{Z}$. Intuitively, by (\ref{eqn: CVaR standard Gaussian}), as $\alpha \to 1^-$, we have $c_\alpha^* \to +\infty$, i.e., the CVaR of the standard Gaussian arm increases without bound, and the arm gets riskier as the user demands higher confidence. However, since $(\sigma c_\alpha^*)/\big(\mu \cdot {\frac{\alpha}{1-\alpha}}\big) \to 0$ as $\alpha \to 1^-$, we have that $\frac{\alpha}{1-\alpha} \to +\infty$ faster than $c_\alpha^*$. Thus, for $\alpha$ sufficiently close to $1^-$, $c_\alpha(X) \approx \mu \big(\frac{\alpha}{1-\alpha}\big)$ and the problem reduces to a standard (risk-neutral, cost minimization or reward maximization) $K$-armed MAB problem. 

Consider a $K$-armed MAB $\nu = \seq{\nu(i)}_{i \in [K]}$ and a player with a \vt{known} risk threshold $\tau>0$ that represents her risk appetite. The CVaR MAB problem is played over a horizon of length $n$. Roughly speaking, our goal is to choose the arm with the lowest average loss, subject to an upper bound on the risk (measured by the CVaR) associated with the arm. For any arm $i$, the loss associated with arm $i$ has distribution $\nu(i)$ and $X(i)$ denotes a   random variable with distribution $\nu(i)$. Furthermore, $\mu(i) = \mu_i$ denotes the mean of $X(i)$, and $c_\alpha(i)$ denotes the CVaR of $X(i)$, \vt{as defined in~\eqref{eqn: CVaR Gaussian}}. Similar to \citet{kagrecha2020constrained}, we define feasible and infeasible instances as follows.
\begin{definition}[\citet{kagrecha2020constrained}]
An \emph{instance} of the risk-constrained MAB problem is defined by $(\nu,\tau)$. We denote the set of feasible arms (whose CVaR $\leq \tau$) as $\kay_\tau = \sett{i \in [K] : c_\alpha(i) \leq \tau}$. The instance $(\nu,\tau)$ is said to be \emph{feasible} (resp.\ \emph{infeasible}) if $\kay_\tau \neq \emptyset$ (resp.\ $\kay_\tau = \emptyset$). 
\end{definition}
In a feasible instance, an arm $i$ is \emph{optimal} if $c_\alpha(i) \leq \tau$ and $\mu_i = \min_{j \in \kay_\tau}\mu_j$. Suppose arm $1$ is optimal and  $\argmin_{i \in \kay_\tau} \mu_i = \sett{1}$ for simplicity. Define the set ${\cal M} := \sett{i\in  \sdiff{[K]}{\sett{1}}:\mu_i > \mu_1}$. 
Arm $i$ is said to be $$\begin{cases}
 \text{\emph{a suboptimal arm}} & \text{if $i\in{\cal M} \cap \kay_\tau$,}\\
 \text{\emph{an infeasible arm}} & \text{if $i\in \kay_\tau^c$,}\\
 \text{\emph{a deceiver arm}} & \text{if $i\in{\cal M}^c \cap \kay_\tau^c$.}
 \end{cases}$$
 For a suboptimal arm $i$, we define the \emph{suboptimality gap} as $\Delta(i) = \mu_i - \mu_1 > 0$. For an infeasible arm $i$, we define the \emph{infeasibility gap} as $\Delta_{\tau} (i,\alpha) = \CVaR{\alpha}{i} - \tau > 0$.
 
 In an infeasible instance, without loss of generality, set arm $1$ as an \emph{optimal} arm, that is, $\CVaR{\alpha}{1} = \min_{i \in [K]} \CVaR{\alpha}{i}$. For simplicity we suppose $\argmin_{i \in [K]} \CVaR{\alpha}{i} = \sett{1}$. We define the \emph{risk gap} for an arm $i$ that is not optimal by $\Delta_r(i,\alpha) = c_{\alpha}(i) - c_{\alpha}(1) > 0$. We remark that it is in this infeasible instance that the risk gap is defined naturally by \citet{kyn2020cvar}. With the three gaps defined, we can now  define three regrets in different settings.
\begin{definition} \label{def:regrets}
Let $T_{i,n}$ denote the number of times arm $i$ was pulled in the first $n$ rounds.
\begin{enumerate}
	\item For a feasible instance, let $$\kay^* = \sett{ i \in [K] : c_\alpha(i) \leq \tau\ \text{and}\ \mu_i = \min_{j \in \kay_\tau} \mu_j}$$ denote the set of optimal arms (and without loss of generality suppose $1 \in \kay^*$). The \emph{suboptimality regret} of policy $\pi$ over $n$ rounds is
	\begin{equation}
		\subreg{n}{\pi} = \sum_{i \in 
		\sdiff{\kay_{\tau}}{\kay^*}} \EE[T_{i,n}]\subgap{i} , \label{eqn:reg1}
	\end{equation}
	and the \emph{infeasibility regret} of policy $\pi$ over $n$ rounds is
	\begin{equation}
		\infreg{n}{\pi} = \sum_{i \in 
		\kay_{\tau}^c}  \EE[T_{i,n}]\infgap{\tau}{i,\alpha}. \label{eqn:reg2}
	\end{equation}
	\item For an infeasible instance, let $$\kay^* = \sett{ i \in [K] : c_\alpha(i) = \min_{j \in [K]} c_\alpha(j)}$$ denote the set of optimal arms (and without loss of generality suppose $1 \in \kay^*$). The \emph{risk regret} of policy $\pi$ over $n$ rounds is
	\begin{equation}
		\riskreg{n}{\pi} = \sum_{i \in \sdiff{[K]}{\kay^*}}  \EE[T_{i,n}]\riskgap{i}. \label{eqn:reg3}
	\end{equation}
\end{enumerate}
\end{definition}
In the following, we design and analyse  CVaR-TS which aims to simultaneously minimize the three regrets in \eqref{eqn:reg1}, \eqref{eqn:reg2}, and~\eqref{eqn:reg3}. To assess the optimality of CVaR-TS, we define the notion of consistency in this risk-constrained setting. 
\begin{definition}
A policy $\pi$ is {\em suboptimality-consistent} if  $\subreg{n}{\pi}=o(n^a)$ for any $a>0$. Similarly, $\pi$ is {\em infeasibility-consistent} (resp.\ {\em risk-consistent}) if $\infreg{n}{\pi}=o(n^a)$ (resp.\ $\riskreg{n}{\pi} =o(n^a)$) for any $a>0$.
\end{definition}
\section{The CVaR-TS Algorithm}
\label{algorithms}

\jc{In this section, we introduce the CVaR Thompson Sampling (CVaR-TS) algorithm for Gaussian bandits with bounded variances, i.e., $\nu \in {\cal E}_{\cal N}^K(\sigma_{\max}^2)$ where 
\begin{align*}
 {\cal E}_{\cal N}^K(\sigma_{\max}^2) &= \big\{\nu = (\nu_1,\dots,\nu_K) : \\
 &\qquad  \nu_i \Fol {\cal N}(\mu_i,\sigma_i^2), \sigma_i^2 \leq \sigma_{\max}^2, \forall\, i \in[ K] \big\}
\end{align*} for some $\sigma_{\max}^2 > 1$.}
Instead of choosing the arm based on the  optimism in the face of uncertainty principle as in \citet{kagrecha2020constrained}, the algorithm samples from the posteriors of each arm, then chooses the arm according to a multi-criterion procedure.

As is well known, a crucial step of Thompson sampling algorithms is the updating of parameters based on Bayes rule. Denote the mean and precision of the Gaussian by $\mu$ and $\phi$ respectively. If $(\mu,\phi) \Fol \mathrm{Normal}\text{-}\mathrm{Gamma}(\mu,T,\alpha, \beta)$, then $\phi \Fol \mathrm{Gamma}(\alpha,\beta)$, and $\mu | \phi \Fol \NDist(\mu,1/(\phi T))$. Since the Normal-Gamma distribution is the conjugate prior for the Gaussian with unknown mean and variance, we use Algorithm~\ref{alg: update} to update $(\mu,\phi)$.

\begin{algorithm}[ht]
   \caption{$\Update(\hat{\mu}_{i,t-1}, T_{i,t-1}, \alpha_{i,t-1}, \beta_{i,t-1})$}
   \label{alg: update}
\begin{algorithmic}[1]
	\STATE {\bfseries Input:} Prior parameters $(\hat{\mu}_{i,t-1}$, $T_{i,t-1}$, $\alpha_{i,t-1}$, $\beta_{i,t-1})$ and new sample $X_{i,t}$
	\STATE Update the mean: $\hat{\mu}_{i,t} = \frac{T_{i,t-1}}{T_{i,t-1}+1}\hat{\mu}_{i,t-1} + \frac{1}{T_{i,t-1}+1} X_{i,t}$
	\STATE Update the number of samples, the shape parameter, and the rate parameter: $T_{i,t} = T_{i,t-1} + 1$, $\alpha_{i,t} = \alpha_{i,t-1} + \frac{1}{2}$, $\beta_{i,t} = \beta_{i,t-1} + \frac{T_{i,t-1}}{T_{i,t-1}+1} \cdot \frac{{(X_{i,t} - \hat{\mu}_{i,t-1})}^2}{2}$
\end{algorithmic}
\end{algorithm}
We present a Thompson sampling-based algorithm to solve the CVaR Gaussian MAB problem. The player chooses a prior over the set of feasible bandits parameters for both the mean and precision. In each round $t$, for each arm $i$, the player samples a pair of parameters $(\theta_{it},\kappa_{it})$ from the posterior distribution of arm $i$, then forms the set $$\hat{\kay}_t := \sett{k \in [K]:\hat{c}_\alpha(k,t) = \theta_{kt} \Big(\frac{\alpha}{1-\alpha}\Big) \!+ \!\frac{1}{\sqrt{\kappa_{kt}}} c_\alpha^* \!\leq\! \tau}.$$ If $\hat{\kay}_t$ is nonempty, i.e., it is plausible that there are some feasible arms available, choose arm $j$ if $\theta_{jt} = \min_{k \in \hat\kay_t} \theta_{kt}$. Otherwise, choose arm $j$ if $j = \argmin_{k \in [k]} \hat{c}_\alpha(k,t)$; that is, choose the least infeasible arm available. At the end of the algorithm, we also set a \texttt{FeasibilityFlag} that checks if the instance is feasible. \citet{kagrecha2020constrained} provided  bounds on the probability of incorrect flagging by any consistent algorithm, and we have empirically compared the errors induced by both algorithms in Section~\ref{numerical_simulations}. \vt{We find that for a moderate horizon $n\approx 1000$,} CVaR-TS does not perform worse than RC-LCB in this aspect; yet regret-wise, CVaR-TS often performs much better.  We provide an explanation for this phenomenon in Section~\ref{numerical_simulations}.
\begin{algorithm}[ht]
   \caption{CVaR Thompson Sampling (CVaR-TS)}
   \label{alg: CVaR-TS}
\begin{algorithmic}[1]
   \STATE {\bfseries Input:} Threshold $\tau$, constant $\alpha$, $\hat{\mu}_{i,0} = 0$, $T_{i,0} = 0$, $\alpha_{i,0} = \frac{1}{2}$, $\beta_{i,0} = \frac{1}{2}$
   \FOR{$t=1,2,\ldots, K$}
   \STATE Play arm $t$ and update $\hat{\mu}_{t,t} = X_{t,t}$
   \STATE $\Update(\hat{\mu}_{t,t-1}, T_{t,t-1}, \alpha_{t,t-1}, \beta_{t,t-1})$
   \ENDFOR
   \FOR{$t=K+1,K+2,...$}
   \STATE Sample $\kappa_{i,t}$ from $\mathrm{Gamma}(\alpha_{i,t-1},\beta_{i,t-1})$
   \STATE Sample $\theta_{i,t}$ from $\mathcal{N}(\hat{\mu}_{i,t-1}, 1/T_{i,t-1})$
   \STATE Set $\hat{c}_{\alpha}(i,t) = \theta_{i,t}\big(\frac{\alpha}{1-\alpha} \big)+\frac{1}{{\sqrt{\kappa_{i,t}}}} c_{\alpha}^*$
   \STATE Set $\hat{\mathcal{K}}_t = \left\{k : \hat{c}_{\alpha}(k,t) \leq \tau\right\}$
   \IF{$\hat{\mathcal{K}}_t \neq \emptyset$}
   \STATE Play arm $i(t) = \argmin_{k \in \hat{\mathcal{K}}_t} \theta_{i,t}$ and observe loss $X_{i(t),t}\sim \nu(i(t))$
   \STATE $\Update(\hat{\mu}_{i(t),t-1}, T_{i(t),t-1}, \alpha_{i(t),t-1}, \beta_{i(t),t-1})$
   \ELSE 
   \STATE Play arm $i(t) = \argmin_{k \in \hat{\mathcal{K}}_t} \hat{c}_{\alpha}(k,t)$ and observe loss $X_{i(t),t}\sim \nu(i(t))$
   \STATE $\Update(\hat{\mu}_{i(t),t-1}, T_{i(t),t-1}, \alpha_{i(t),t-1}, \beta_{i(t),t-1})$
   \ENDIF
   \ENDFOR
   \IF{$\hat{\kay}_t \neq \emptyset$}
   \STATE Set $\texttt{FeasibilityFlag = true}$
   \ELSE
   \STATE Set $\texttt{FeasibilityFlag = false}$
   \ENDIF
\end{algorithmic}
\end{algorithm}
\section{Regret Bounds of CVaR-TS and Lower Bounds}
\label{regret_analyses}

We present our regret bounds in the following theorems. We then compare our bounds to those of other competing algorithms. Most comparisons of results are with respect to Theorem~\ref{thm: risk_regret}, since a natural definition of regret induced by the CVaR risk measure is considered in \citet{kyn2020cvar}, \citet{tamkin2020dist}, \citet{xi2020nearoptimal}, and~\citet{soma2020statistical}. \vt{We also compare these bounds to instance-dependent lower bounds \citep[Theorem~4]{kagrecha2020constrained} that we particularize to Gaussian bandits.}

\begin{theorem}
\label{thm: risk_regret}
Fix $\xi \in (0,1)$, $\alpha \in (1/2,1)$. In an infeasible instance, the asymptotic expected risk regret of CVaR-TS for CVaR Gaussian bandits satisfies
	\begin{align}
	\limsup_{n \to \infty}\frac{\riskreg{n}{\text{CVaR-TS}}}{\log n} \leq \sum_{i \in \sdiff{[K]}{\kay^*}} C_{\alpha,\xi}^i\riskgap{i},\nonumber
	\end{align}
	where $C_{\alpha,\xi}^i =\max\big\{A_{\alpha,\xi}^i,B_{\alpha,\xi}^i \big\}$, 
	\begin{align}
		\label{risk_coeff_A}A_{\alpha,\xi}^i&=\frac{2\alpha^2}{\xi^2{(1-\alpha)}^2\Delta_r^2(i,\alpha)},\\
		B_{\alpha,\xi}^i&=\frac{1}{h \paren{\frac{\sigma_i^2 {(c_\alpha^*)}^2}{{\paren{\sigma_i c_\alpha^* -(1-\xi)\riskgap{i}}}^2}}},\quad \mbox{and}\label{eqn:def_B}\\
		h(x) &= \tsfrac{1}{2}(x-1-\log x).\nonumber
	\end{align}
	Furthermore, choosing 
\begin{equation}
	\xi_\alpha = 1 - \frac{\sigma_i c_\alpha^*}{\riskgap{i}}\paren{1 - \frac{1}{\sqrt{\inv{h_+}(1/A_{\alpha,1}^i)}}},\label{eqn:xi_choice1}
	\end{equation}	
 where $\inv{h_+}(y) = \max \sett{x : h(x) = y} \geq 1$, yields $B_{\alpha,\xi_\alpha}^i \leq A_{\alpha,\xi_\alpha}^i$ and $\xi_\alpha \to 1^-$ as $\alpha \to 1^-$.
\end{theorem}
\begin{remark} {\em
	\vt{The final part of the theorem claims that the upper bound is characterised by $A_{\alpha,\xi_\alpha}^i$.} By continuity, we obtain the regret bound involving $A_{\alpha,1}^i$ as defined in (\ref{risk_coeff_A}) with $\xi_\alpha \to 1^-$ given by (\ref{eqn:xi_choice1}). We remark that this asymptotic regret bound is tighter than several existing results under certain regimes, as summarised in Table~\ref{tab:risk_regret}, which lists the upper bounds on the expected number of pulls $\EE[T_{i,n}]$ of a non-optimal arm $i$ over a horizon $n$ by various policies $\pi$. 
	\begin{table}
	\begin{center}
	\begin{tabular}{| c || c | c |}
	\hline
	 Paper & Expected number of pulls $\mathbb{E}[T_{i,n}]$ & Conditions\\\hline
	 (1) & $\frac{16 \log(Cn)}{\beta^2\Delta_r^2(i,\alpha)} + K (1 + \frac{\pi^2}{3})$ & $\alpha^2 \leq 8$\\\hline
	 (2) & $\frac{2\gamma \log n}{\beta^2 \Delta_r^2(i,\alpha)} + \frac{\gamma}{2} \paren{\frac{1}{\gamma-2} + \frac{3}{\gamma - 20\beta}}$ & $\alpha^2 \leq \gamma$\\\hline
	(3) & $\frac{4U^2 \log(\sqrt{2}n)}{\beta^2 \Delta_r^2(i,\alpha)} + 3$ & $\alpha^2 \leq 2U^2$\\\hline
	(4) & $\frac{4 \log (2D_\sigma n^2)}{\beta^2 \Delta_r^2(i,\alpha)d_\sigma} + K + 2$ & $\alpha^2 \leq \frac{4}{d_\sigma}$\\\hline
	\end{tabular}
	\caption{\label{tab:risk_regret} Comparison of the  expected regret of CVaR-TS to those of  existing CVaR MAB algorithms: (1) CVaR-LCB by \protect\citet[Theorem~1]{bhat2019concentration}, (2) CVaR-UCB-1 by \protect\citet[Lemma~4]{kyn2020cvar}, (3) CVaR-UCB-2 by \protect\citet{tamkin2020dist}, (4) RC-LCB by \protect\citet[Theorem~2]{kagrecha2020constrained}. We abbreviate $1-\alpha $ by $\beta$. The second column corresponds the expected number of pulls of a non-optimal arm $i$ over horizon $n$, which suffices to provide comparisons on the expected regret bounds. The last column states   conditions under which CVaR-TS performs better than the algorithm in comparison. See  Remark~\ref{table_discussions} for details.}
	\end{center}
	\end{table} 
	\vspace{-.03in}
	}
\end{remark}
\begin{remark} {\em 
For any arm $i$, $\lim_{\alpha \to 1^-}(1-\alpha)\Delta_r(i,\alpha) = \mu_i - \mu_1$, and the upper bound simplifies to ${2}/{{(\mu_i-\mu_1)}^2}$. This agrees with our intuition because as $\alpha\to 1^-$, $c_{\alpha}(i) = \mu_i \big( \frac{\alpha}{1-\alpha} \big) + \sigma_i c_\alpha^*$ is dominated by $\mu_i \big( \frac{\alpha}{1-\alpha} \big)$, implying that we are  in the risk-neutral setting. Thus, the results are analogous to those derived for mean-variance bandits \citep{zhu2020thompson} for the risk-neutral setting when $\rho \to +\infty$ (recall the mean-variance of arm $i$ is $\mathrm{MV}_i =\rho\mu_i-\sigma_i^2$). 
}
\end{remark}

\begin{remark}
\label{table_discussions}
{\em From Table~\ref{tab:risk_regret}, we see that CVaR-TS outperforms existing state-of-the-art CVaR MAB algorithms on Gaussian bandits under certain regimes.
\begin{enumerate}
	\item \label{remark_1}We see that CVaR-TS outperforms CVaR-LCB \citep{bhat2019concentration} unconditionally since $\alpha \in (0,1)$ in Definition~\ref{def: CVaR}, and hence $\alpha^2 < 1 \leq 8$ trivially. Replacing $8$ in Remark~\ref{table_discussions}.\ref{remark_1} with $2 \leq \gamma$, where $\gamma $ is the UCB parameter in \cite{kyn2020cvar} (denoted as $\alpha$ therein), yields the conclusion that CVaR-TS outperforms CVaR-UCB-1 unconditionally.
	\item Our regret bound for CVaR-TS is tighter than that of CVaR-UCB-2 in~\citet{tamkin2020dist} when $\alpha^2 \leq 2U^2$, where $U > 0$ is any upper bound on the supports of the distributions  of the $K$ bandits. Notice that we do not need to assume that the arm distributions are bounded \vt{(since our arm distributions are Gaussian)}, which is different from~\citet{tamkin2020dist} (and also the contemporaneous work of \citet{baudry2020thompson}).
	\item  Our regret bound for CVaR-TS is tighter than that of RC-LCB when $\alpha^2 \leq {4}/{d_\sigma}$, where $\sigma$ is the fixed sub-Gaussianity parameter of the bandits. We remark that \citet[Lemma~1]{kagrecha2020constrained} included the implicit constants $D_\sigma$ and $d_\sigma$ in their algorithm design due to an LCB-style concentration bound derived from  \citet[Corollary~1]{bhat2019concentration}, which in turn was derived from a concentration of measure result involving  the Wasserstein distance in \citet{fournier2015rate}. Note that $d_\sigma$ is non-trivial to compute \vt{as it is implicitly stated}, so we estimate it to be  ${1}/(8\sigma^2)$ based on an explicit concentration bound for $d_\sigma$ in   \citet[Theorem~3.1]{a2019concentration}. \vt{With this choice}, CVaR-TS outperforms RC-LCB when the sub-Gaussianity parameter  of the $K$ bandits  $\sigma$ satisfies $\sigma^2 \geq {\alpha^2}/{32}$.
\end{enumerate}}
\end{remark}
\begin{theorem}
\label{thm: inf_regret}
Fix $\xi \in (0,1)$, $\alpha \in (1/2,1)$. In a feasible instance, the asymptotic expected infeasibility regret of CVaR-TS for CVaR Gaussian bandits satisfies
	\begin{align}
	\limsup_{n \to \infty}\frac{\infreg{n}{\text{CVaR-TS}}}{\log n} \leq \sum_{i \in \kay_\tau^c} F_{\alpha,\xi}^i\infgap{\tau}{i,\alpha},\nonumber
	\end{align}
	where $F_{\alpha,\xi}^i=\max\big\{D_{\alpha,\xi}^i,E_{\alpha,\xi}^i\big\}$,  
	\begin{align}
	\label{inf_regret_E} D_{\alpha,\xi}^i&=\frac{2\alpha^2}{\xi^2{(1-\alpha)}^2\Delta_{\tau}^2(i,\alpha)},\quad\mbox{and}\\
	E_{\alpha,\xi}^i&=\frac{1}{h \paren{\frac{\sigma_i^2 {(c_\alpha^*)}^2}{{\paren{\sigma_i c_\alpha^* -(1-\xi)\infgap{\tau}{i,\alpha}}}^2}}}. \nonumber
	\end{align}
	Furthermore, choosing 
\begin{equation}
	\xi_\alpha = 1 - \frac{\sigma_i c_\alpha^*}{\infgap{\tau}{i,\alpha}}\paren{1 - \frac{1}{\sqrt{\inv{h_+}(1/D_{\alpha,1}^i)}}}\label{eqn:xi_alpha2}
	\end{equation}	
 yields $E_{\alpha,\xi_\alpha}^i \leq D_{\alpha,\xi_\alpha}^i$ and $\xi_\alpha \to 1^-$ as $\alpha \to 1^-$.
\end{theorem}
\begin{remark}{\em 
\vt{As with Theorem~\ref{thm: risk_regret}, the final part of the theorem says that the upper bound is characterised by $D_{\alpha,\xi_\alpha}^i$.} By continuity, we obtain the regret bound of $D_{\alpha,1}^i$ as defined in (\ref{inf_regret_E}) with $\xi_\alpha \to 1^-$ given by (\ref{eqn:xi_alpha2}).
This is expected, since $(c_\alpha(1),\riskgap{i})$ can be replaced by $(\tau,\infgap{\tau}{i,\alpha})$ analogously for infeasible arms in the infeasible case, and the computations can be similarly reused.
Likewise, this regret bound is tighter than RC-LCB  (\citet[Theorem 2]{kagrecha2020constrained}, $\EE[T_{i,n}] \leq  {4 \log (2D_\sigma n^2)}/{(\beta^2 \Delta_{\tau}^2(i,\alpha)d_\sigma)}$) provided that $d_\sigma \leq 4$.}
\end{remark}
\begin{theorem}
\label{thm: sub_regret}
In a feasible instance, the asymptotic expected suboptimality regret of CVaR-TS for CVaR Gaussian bandits satisfies
	\begin{align}
	\limsup_{n \to \infty}\frac{\subreg{n}{\text{CVaR-TS}}}{\log n} \leq \sum_{i \in \sdiff{\kay_\tau}{\kay^*}} \frac{2}{\subgap{i}}.
	\end{align}
\end{theorem}
\begin{remark}
{\em This is expected, since the problem reduces to a standard MAB. Our empirical simulations in Section \ref{numerical_simulations} corroborate this result.  Additionally, our regret bound is tighter than RC-LCB (in which $\EE[T_{i,n}] \leq {(16 \sigma^2 \log n)}/{\Delta^2(i)}$) under the condition that $\sigma^2\geq {1}/{8}$ (where $\sigma$ is the largest sub-Gaussianity parameter of the arms).}
\end{remark}
\begin{remark}  
{\em For an infeasible and suboptimal arm $i$ in a feasible instance, the expected number of pulls us upper bounded by ${2}/{\Delta^2(i)}$ and $D_{\alpha,\xi}^i$, and thus can be   written as $$\limsup_{n \to \infty}\frac{\EE[T_{i,n}]}{\log n} \leq \min\sett{\frac{2}{\Delta^2(i)},D_{\alpha,\xi}^i}.$$ This agrees with the regret analysis of RC-LCB  if we take $\alpha\to 1^-$ and choose $\xi_\alpha\to 1^-$ according to \eqref{eqn:xi_alpha2}.}
\end{remark}
\vt{\begin{theorem}[Lower Bounds]
\label{thm: lower_bounds}
Fix \jc{a risk threshold $\tau \in \RR$,} $\xi \in (0,1)$ and $\alpha \in (1/2,1)$ Recall the definitions of $A_{\alpha,\xi}^i$ and $D_{\alpha,\xi}^i$ in \eqref{risk_coeff_A} and \eqref{inf_regret_E} respectively. 
\begin{enumerate}
	\item In an infeasible instance \jc{$(\nu,\tau)$, where $\nu\in{\cal E}_{\cal N}^K(\sigma_{\max}^2)$, for any risk-consistent policy $\pi$,} 
	$$
	\liminf_{n \to \infty} \frac{\riskreg{n}{\pi}}{\log n} \geq \sum_{i \in \sdiff{[K]}{\kay^*}} \xi^2 \sigma_i^2 A_{\alpha,\xi}^i\riskgap{i}.
	$$
	\item In a feasible instance \jc{$(\nu,\tau)$, where $\nu\in{\cal E}_{\cal N}^K(\sigma_{\max}^2)$, for any infeasibility-consistent policy $\pi$,} 
	$$\liminf_{n \to \infty} \frac{\infreg{n}{\pi}}{\log n} \geq \sum_{i \in \kay_\tau^c} \xi^2 \sigma_i^2 D_{\alpha,\xi}^i\infgap{\tau}{i}.$$
	\jc{Also, for any suboptimality-consistent policy $\pi$,}
	$$\liminf_{n \to \infty} \frac{\subreg{n}{\pi}}{\log n} \geq \sum_{i \in \sdiff{\kay_\tau}{\kay^*}} \frac{2}{\Delta(i)}.$$
\end{enumerate}
\end{theorem}
\begin{remark} \label{rmk:lb}
\em We see that the asymptotic lower bound for the suboptimality regret matches its asymptotic upper bound in Theorem~\ref{thm: sub_regret}. Furthermore, if the standard deviations of the Gaussian arms satisfy $1 \leq \sigma_1 \leq \sigma_i \leq \sigma_{\max}$ for all $i \in [K]$ and $\alpha \to 1^-$, together with the appropriate choices of $\xi_\alpha$ in~\eqref{eqn:xi_choice1} and~\eqref{eqn:xi_alpha2}, we see that the asymptotic lower bounds in Parts~1 and~2 of Theorem~\ref{thm: lower_bounds} match their respective asymptotic upper bounds in Theorems~\ref{thm: risk_regret} and~\ref{thm: inf_regret}.  Thus,  CVaR-TS is {\em asymptotically optimal} in the risk regret (Theorem \ref{thm: risk_regret}) and infeasiblity regret (Theorem \ref{thm: inf_regret}) if $1 \leq \sigma_1 \leq \sigma_i \leq \sigma_{\max}$. CVaR-TS is {\em asymptotically optimal} in the suboptimality regret (Theorem \ref{thm: sub_regret})  
unconditionally.
\end{remark}
The main idea in the proof of Theorem \ref{thm: lower_bounds} is to particularize the instance-dependent lower bounds in \citet[Theorem~4]{kagrecha2020constrained} for  Gaussian CVaR MABs. For Part~1, the vanilla lower bound is stated in terms of  the KL-divergence between the distribution of arm $i$ infimized over the distribution of an alternative arm whose Gaussian CVaR is less than $c_\alpha^*$.  We judiciously choose the distribution of the alternative arm $\nu'(i)$ to be ${\cal N}\paren{\mu_i - \frac{1}{\xi}{\sqrt{2/A_{\alpha,\xi}^i}}-\eeps,\sigma_i^2}$ for each $\eeps > 0$, which then yields the almost-optimal lower bound in view of Remark~\ref{rmk:lb}. 
}
\section{Proof Outline of Theorem~\ref{thm: risk_regret} and Extensions}
\label{proof_outline}

We outline the proof of Theorem~\ref{thm: risk_regret}, since it is the most involved. Theorem~\ref{thm: inf_regret} follows by a straightforward substitution and Theorem~\ref{thm: sub_regret} is analogous to the proof of Thompson sampling for  standard MAB. Denote the \vt{sample CVaR (the Thompson sample)} as $\sCVaR{\alpha}{i,t} = \theta_{i,t}\Big( \frac{\alpha}{1-\alpha} \Big) + \frac{1}{\sqrt{\kappa_{i,t}}} c_\alpha^*$. Fix $\eeps > 0$ and define $E_i(t) := \sett{\sCVaR{\alpha}{i,t} > \CVaR{\alpha}{1} + \eeps}$
the event that the Thompson sample mean of arm $i$ is $\eeps$-riskier than a certain threshold or, more precisely, $\eeps$-higher than the optimal arm (which has the lowest CVaR, quantifying the expected loss). Intuitively, event $E_i(t)$ is highly likely to occur when the algorithm has explored sufficiently. However, the algorithm does not choose arm $i$ when $E_i^c(t)$ occurs with small probability under Thompson sampling. We can split $\EE[T_{i,n}]$ into two parts as follows.
\begin{lemma}[\citet{lattimore_szepesvari_2020}]
\label{subthm: lattimore_main}
Let $\PP_t(\, \cdot \, ) = \PP(\, \cdot \,  | A_1,X_1,\dots,A_{t-1},X_{t-1})$ be the probability measure conditioned on the history up to time $t-1$ and $G_{is} = \PP_t(E_i^c(t) | T_{i,t} = s)$, where $E_i(t)$ is any specified event for arm $i$ at time $t$. Denote $\Lambda_1 = \EE \big[  \sum_{s=0}^{n-1} \big\{\frac{1}{G_{1s}} - 1\big\}\big]$ and $\Lambda_2 =\sum_{s=0}^{n-1} \PP \paren{G_{is} > \frac{1}{n}}$. Then $\EE[T_{i,n}] \leq \Lambda_1 + \Lambda_2 + 1$.
\end{lemma}
It remains to upper bound $\Lambda_1$ and $\Lambda_2 $. The techniques used to upper-bound $\Lambda_1$ resemble those in~\citet{zhu2020thompson}, and are relegated to the supplementary material. To upper bound $\Lambda_2$, we split the event $E_i^c(t) = \sett{\hat c_\alpha(i,t) \leq c_\alpha(1) + \eeps}$ into 
\begin{align*}
\Psi_1(\xi) &= \sett{\paren{\theta_{i,t} \!-\! \mu_i}\paren{\tsfrac{\alpha}{1\! -\! \alpha}}\! \leq\! - \xi (\riskgap{i}\! -\! \eeps)},\\
\Psi_2(\xi) &= \sett{\paren{\tsfrac{1}{\sqrt{\kappa_{i,t}}}-\sigma_i} c_\alpha^* \leq (-1+\xi)(\riskgap{i}-\eeps)}.
\end{align*}
That is, $E_i^c(t)  \subseteq \Psi_1(\xi) \cup \Psi_2(\xi)$. We then use the union bound which yields $\PP(E_i^c(t) ) \leq \PP(\Psi_1(\xi)) + \PP(\Psi_2(\xi))$ and concentration bounds to upper bound $\PP(\Psi_1(\xi))$ and  $\PP(\Psi_2(\xi))$. A good choice of the free parameter $\xi$ allows us to allocate ``weights'' to the bounds on $\PP(\Psi_1(\xi))$ and $\PP(\Psi_2(\xi))$ which then yields $A_{\alpha,\xi}^i$ and $B_{\alpha,\xi}^i$  (in~\eqref{risk_coeff_A} and~\eqref{eqn:def_B}) without incurring further residual terms.

\vspace{-.07in}

\paragraph{Extensions and Improvements to Existing Works} The Thompson sample of the Gaussian CVaR is of the form $a f(\theta) + b g(1/\kappa)$, where $\theta$ and $\kappa$ follow  a Gaussian and  Gamma distribution respectively, $f(x) = x,g(x) = \sqrt{x}$ are bijective on their natural domains and have inverses $\inv f,\inv g$ respectively, and $a = \alpha/(1-\alpha), b=c_\alpha^*$ depend on $\alpha$. By splitting \vt{the Thompson sample $\hat c_\alpha(i,t)$ into two parts} as described in the previous paragraph, and using $\inv f$ and $\inv g$, we reduced the problem to upper bounding probabilities of the form $\PP(\theta \leq \cdot)$ and $\PP(\kappa \geq \cdot)$ and concentration bounds can then be readily applied. This hints that for general risk measures on Gaussian bandits with Thompson samples of the form $a f(\theta) + b g(1/\kappa)$, a similar strategy can be used to establish tight regret bounds. For example, the {\em entropic risk} \citep{lee2020learning, Howard72} with parameter $\gamma$ of a Gaussian random variable $X \Fol {\cal N}(\mu,\sigma^2)$ is $ \frac{1}{\gamma} \log \EE[\exp(-\gamma X)] = - \mu + (\gamma/2)\sigma^2$ with a corresponding Thompson sample   $-\theta + (\gamma/2) (1/\kappa)$, where $(f(x),g(x),a,b) = (x, x, -1, \gamma/2)$, and our techniques are also immediately amenable to the entropic risk. Similarly, the {\em mean-variance} risk measure \citep{lee2020learning, zhu2020thompson} with parameter $\rho$ of $X \Fol {\cal N}(\mu,\sigma^2)$   is $\rho \mu - \sigma^2$, with a corresponding Thompson sample  $\rho \theta - 1/\kappa$, where $(f(x),g(x),a,b) = (x,x,\rho,-1)$.  In fact,  using our techniques and, in particular arguments, similar to Lemmas~\ref{lem: tail_upper_bd} and~\ref{lem: upper_bound_term_2_risk}, we can immediately strengthen the regret bound of MVTS in \citet[Theorem~3]{zhu2020thompson} from being dependent on $O((\mu_i-\mu_1)^{-2})$ to $O(( \mathrm{MV}_i-\mathrm{MV}_1)^{-2})$. If instead $f$ and $g$ take very different forms, establishing regret bounds might be more challenging and require new analytical techniques.
\section{Numerical Simulations}
\label{numerical_simulations}
\begin{figure}[t]
\begin{center}
\includegraphics[scale=\picscale]{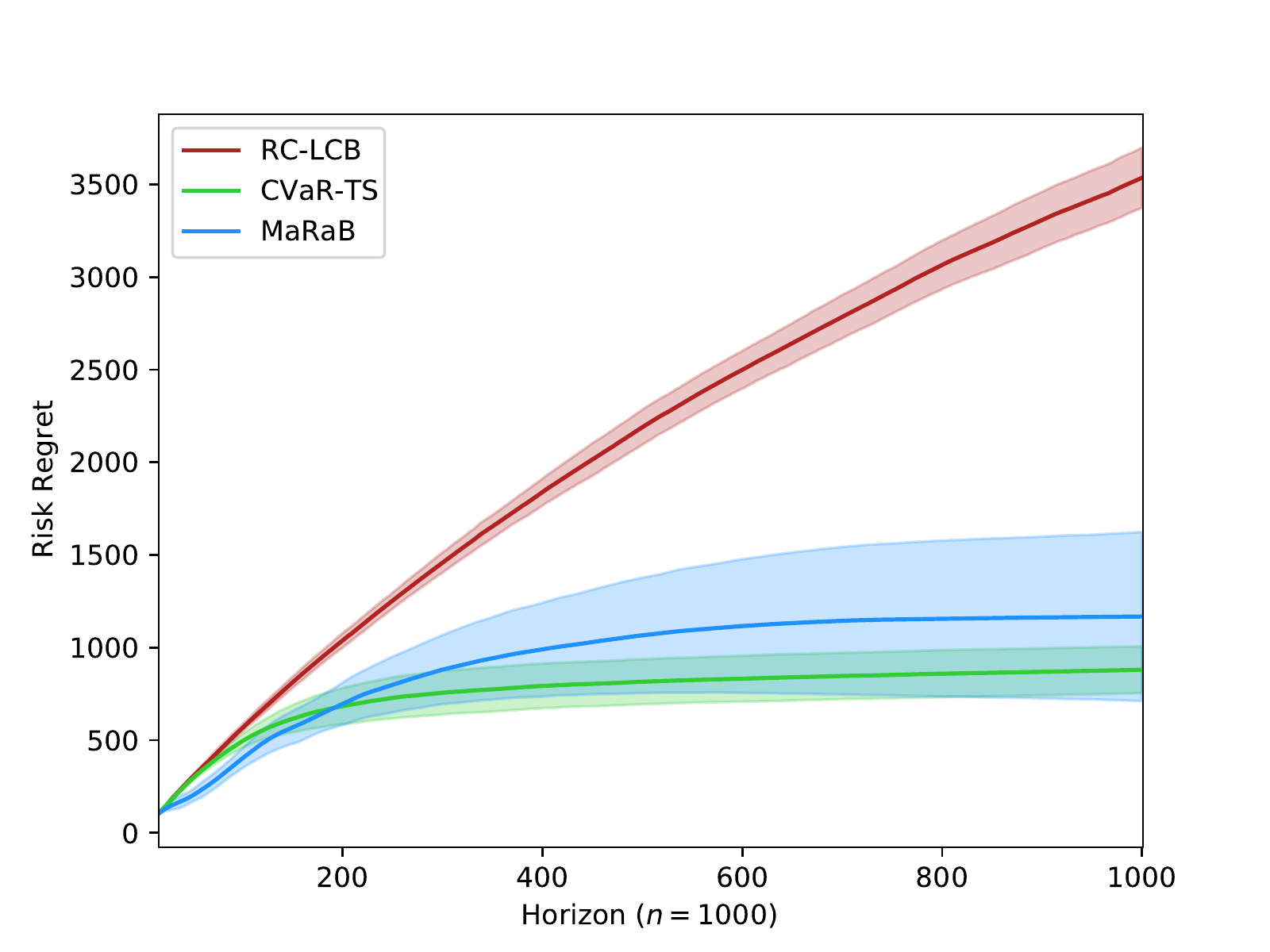}
\includegraphics[scale=\picscale]{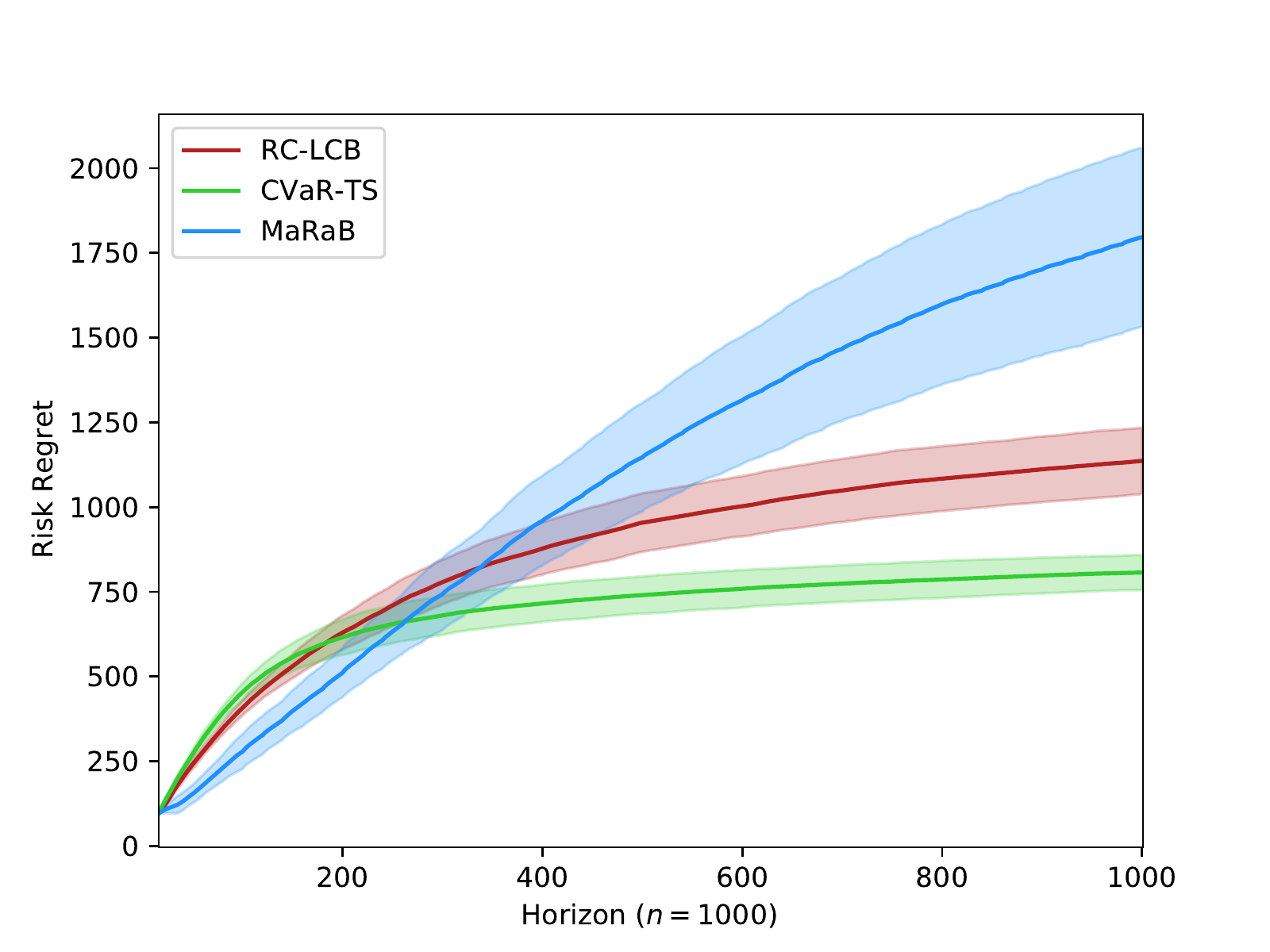}
\end{center}
\caption{Risk regrets averaged over $100$ runs of instances whose arms follow Gaussian distributions with means and variances according to $(\mu,\sigma_>^2)$ and $(\mu,\sigma_\approx^2)$ respectively. The error bars indicate $\pm 1$ standard deviations over the $100$ runs.}
\label{fig: risk_regrets}
\end{figure}
\begin{figure}[t]
\begin{center}
\includegraphics[scale=\picscale]{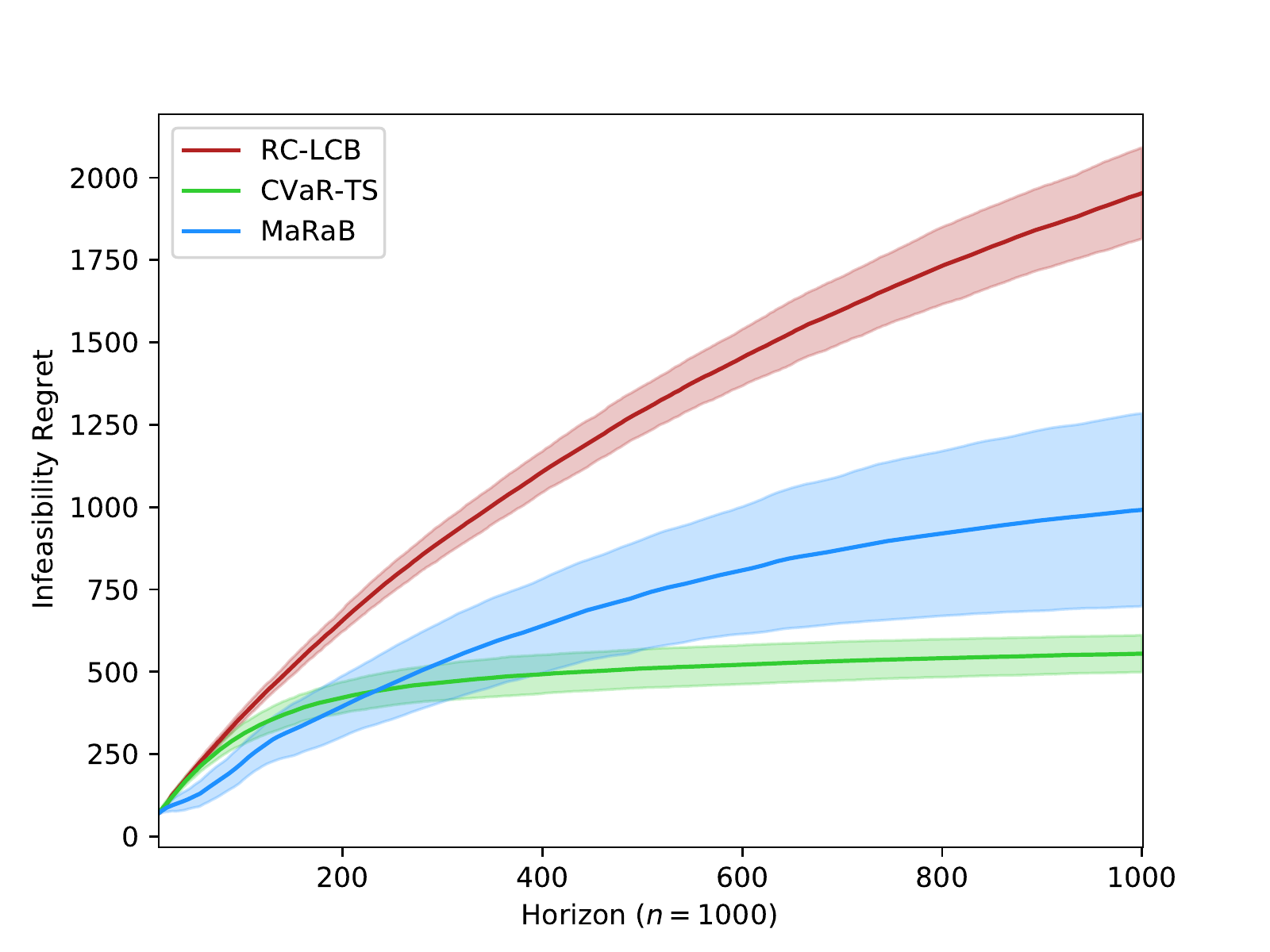}
\includegraphics[scale=\picscale]{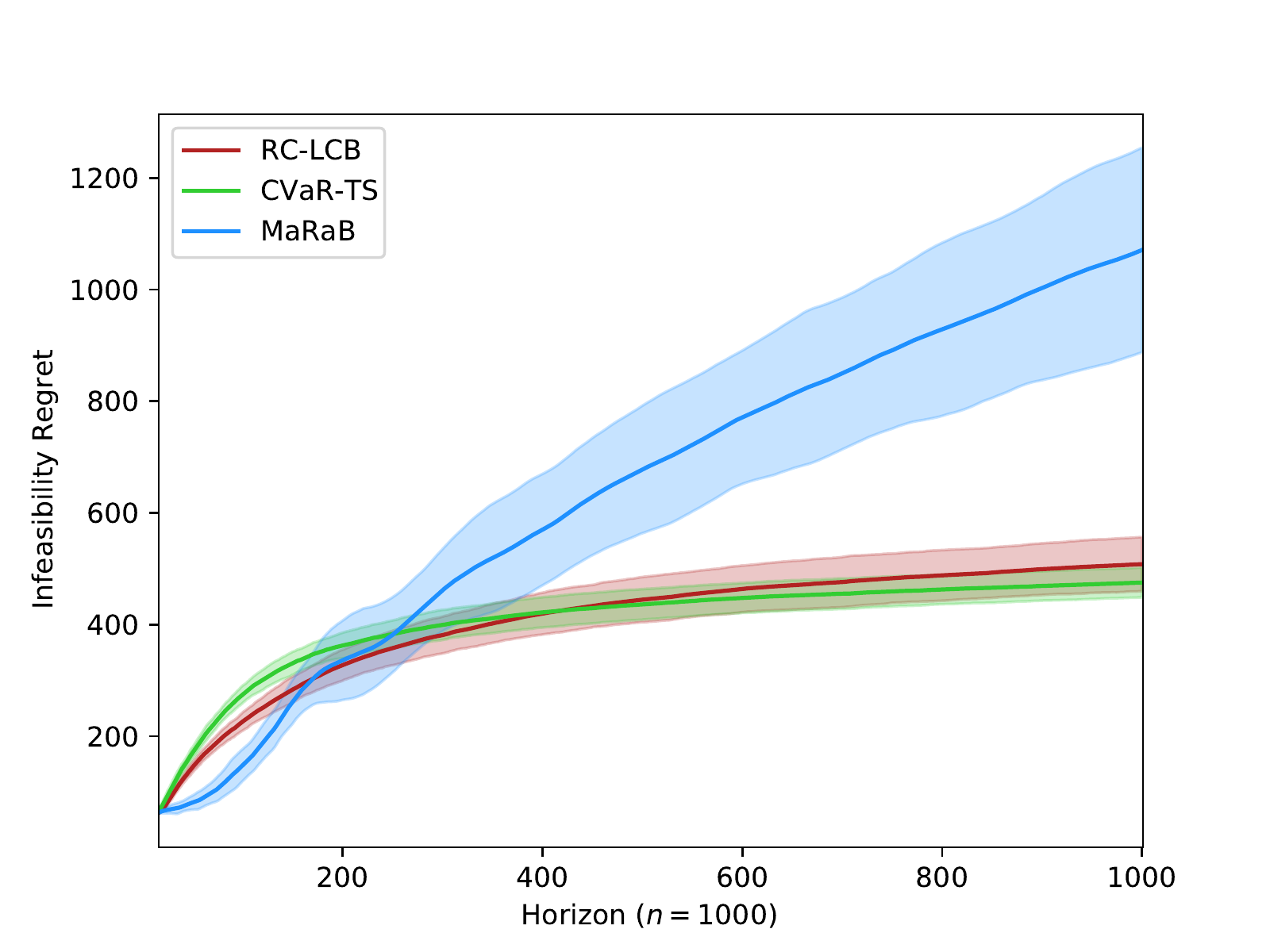}
\end{center}
\caption{Infeasibility regrets averaged over $100$ runs of instances whose arms follow Gaussian distributions with means and variances according to $(\mu,\sigma_>^2)$ and $(\mu,\sigma_\approx^2)$ respectively.}
\label{fig: inf_regrets}
\end{figure}
\begin{figure}[t]
\begin{center}
\includegraphics[scale=\picscale]{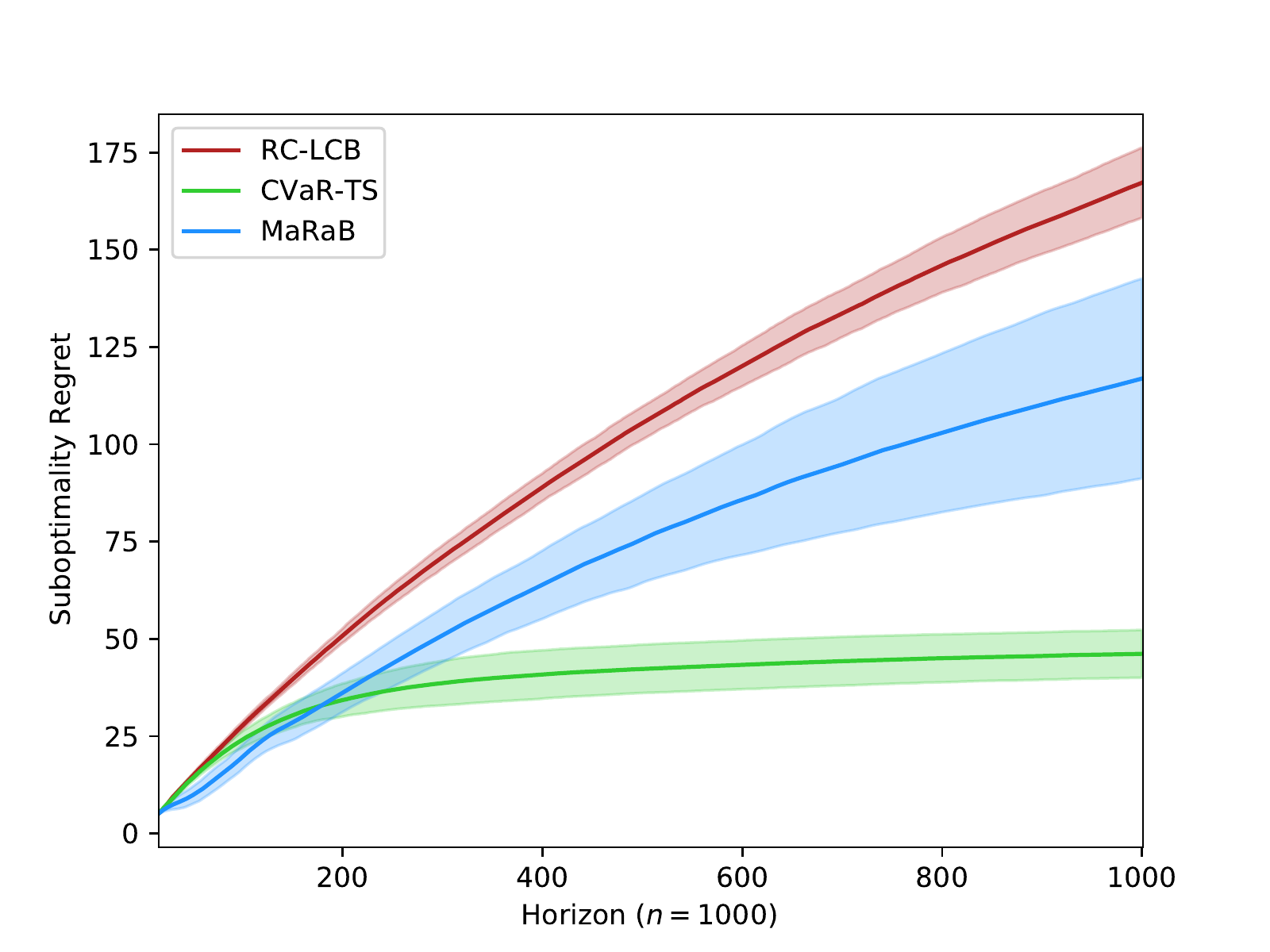} 
\includegraphics[scale=\picscale]{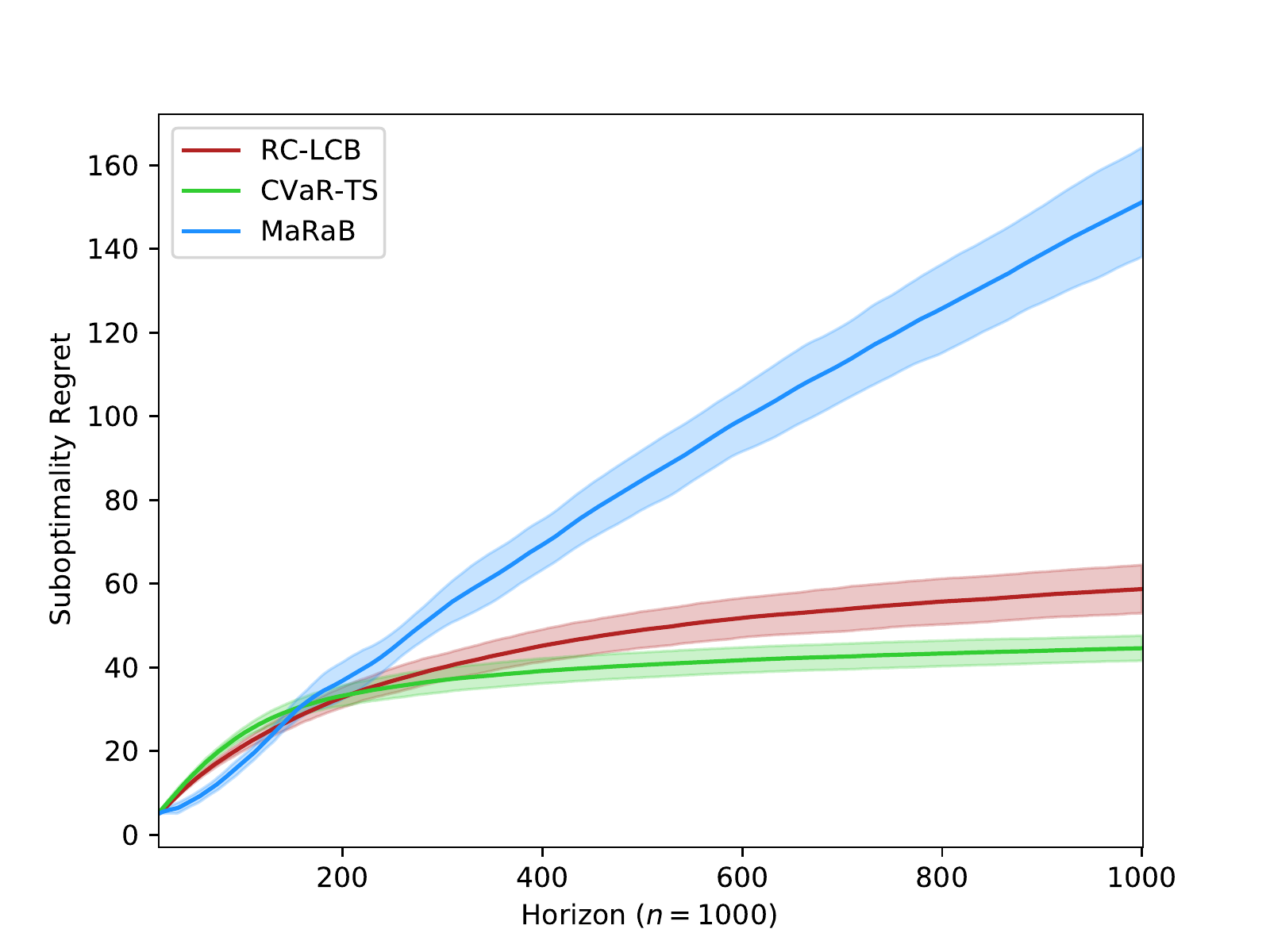}
\end{center}
\caption{Suboptimality regrets averaged over $100$ runs of instances whose arms follow Gaussian distributions with means and variances according to $(\mu,\sigma_>^2)$ and $(\mu,\sigma_\approx^2)$ respectively.}
\label{fig: sub_regrets}
\end{figure}

\vspace{-.07in}

We verify the theory developed using numerical simulations. Set parameters $(N,\alpha, \tau, K, n) = (100,0.95, 4.6, 15,1000)$. We present regret that is averaged over $N$ statistically independent runs of horizon $n$. For the $K = 15$ arms, we set their means as per the experiments in~\citet{sani2013riskaversion} and \citet{zhu2020thompson}, namely, $\mu = (0.1$, $0.2$, $0.23$, $0.27$, $0.32$, $0.32$, $0.34$, $0.41$, $0.43$, $0.54$, $0.55$, $0.56$, $0.67$, $0.71$, $0.79$). 

\vt{By \citet[Theorem 3.1]{a2019concentration} and \citet[Lemma 1]{kagrecha2020constrained}, we have $(D_\sigma,d_\sigma)=(3,{1}/{(8\sigma^2)})$ where we set $\eeps < 2\delta$  therein for simplicity. Under this regime, the regret bounds for CVaR-TS are tighter than those of RC-LCB when the sub-Gaussianity paramater $\sigma$ satisfies $\sigma^2 \geq {1}/{32}=0.03125$. Thus, when investigating Gaussian bandits with variances significantly larger than (resp.\ close to) ${1}/{32}$, we set the variances of the arms according to $\sigma_{>}^2$ (resp.\ $\sigma_{\approx}^2$), where $\sigma_{> }^2 = $ $(0.045$, $0.144$, $0.248$, $0.339$, $0.243$, $0.172$, $0.039$, $0.144$, $0.244$, $0.353$,  $0.244$, $0.146$, $0.056$, $0.149$, $0.285)$ (resp.\ $\sigma_{\approx}^2 = (0.0321$, $0.0332$, $0.0355$, $0.0464$, $0.0375$, $0.0486$, $0.0397$, $0.0398$, \linebreak  $0.0387$, $0.0378$, $0.0567$, $0.0456$, $0.0345$, $0.0334$, $0.0323)$). In this setup, the first $2$ arms are feasible, and arm $1$ is optimal. We additionally set $\tau = 2$   to investigate an infeasible instance. We run  \textsc{MaRaB} \citep{galichet2013exploration}, RC-LCB  \citep{kagrecha2020constrained} and   CVaR-TS. 

Our results  are shown in Figures~\ref{fig: risk_regrets}--\ref{fig: sub_regrets}. Each pair of figures compares the risk, infeasibility and suboptimality regrets (cf.~Definition~\ref{def:regrets}) under the regime with variances $\sigma_>^2$ and $\sigma_\approx^2$. We see a much superior performance by CVaR-TS  for instances with $\sigma^2 \gg {1}/{32}$, corroborating the theoretical conclusion of our regret analyses (summarized in Table~\ref{tab:risk_regret}). When $\sigma^2\approx 1/32$, CVaR-TS still outperforms RC-LCB but not by much. \vt{Compared to \textsc{MaRaB}, CVaR-TS also performs better, sometimes significantly, in terms all the regrets and variances $\sigma_>^2$ and $\sigma_\approx^2$.} The Python code is included in the supplementary material.}

\input{raw/raw_supp/flags.tex}
 \vspace{-.07in}
\section{Conclusion}
\label{conclusion}
\vspace{-.07in}
This paper applies Thompson sampling \citep{thompson1933likelihood} to provide a solution  for CVaR MAB problems  \citep{galichet2013exploration,kagrecha2020statistically, kyn2020cvar} which were largely approached from the L/UCB perspectives previously. The regret bounds are notable improvements over those obtained by the state-of-the-art L/UCB techniques, when the bandits  are Gaussians satisfying certain assumptions. We show that the bounds also coincide with instance-dependent lower bounds when $\alpha\to 1^-$. We corroborated the theoretical results through simulations and verified that under the conditions predicted by the theorems, the gains over previous approaches are significant. 


Noting the similarity of the mean-variance of a Gaussian arm when $\rho \to +\infty$ and the CVaR of the same arm when $\alpha \to 1^-$, we believe there is a unifying theory of risk measures for Gaussian MABs; see \citet{cassel18}, \citet{lee2020learning},  and~\citet{xi2020nearoptimal}. Furthermore, most papers consider sub-Gaussian bandits, while we focus on Gaussian distributions. This is perhaps why we obtain better regret bounds in some regimes. Further work includes analysing Thompson sampling of Gaussian MABs under general risk measures and exploring the performance of Thompson sampling for CVaR sub-Gaussian bandits. 


\newpage
\bibliographystyle{icml2021}
\bibliography{CVaR_TS}

\appendix

\onecolumn
\input{raw/raw_supp/risk_proof_3.tex}
\input{raw/raw_supp/inf_proof_3.tex}
\input{raw/raw_supp/sub_proof.tex}
\input{raw/raw_supp/lower_bound_risk.tex}

\end{document}

%% file: raw/Template.tex
\usepackage{multicol}
\usepackage{ragged2e}
\usepackage{etoolbox}
\usepackage{mathtools}
\usepackage{tikz}
\usepackage{tkz-euclide}
\usepackage{graphicx}
\usepackage{amsmath, amsthm}
\usepackage{hyperref}
\usepackage{footnote}
\makesavenoteenv{tabular}

\newcommand{\picscale}{0.405}

\newtheorem{lemma}{Lemma}

\newtheorem{definition}{Definition}

\newtheorem{remark}{Remark}

\usepackage{xcolor}
\definecolor{blue1}{RGB}{0, 0, 130}
\definecolor{blue2}{RGB}{0, 0, 200}
\definecolor{red2}{RGB}{180,0,0}
\definecolor{purp1}{RGB}{200,25,190}
\definecolor{green1}{RGB}{0,130,0}
\definecolor{lightgray}{RGB}{190,190,190}

\usepackage{times}

\newcommand{\qalph}{\textbf{(\alph*)}}

\usepackage{sectsty}

\newtoggle{cheat}

\newtoggle{midterm}

\newtoggle{defi}
\toggletrue{defi}

\iftoggle{cheat}{\togglefalse{defi}}{}

\newtoggle{key}
\toggletrue{key}

\newcommand{\vocab}[1]{\textit{{#1}}}

\newcommand{\paren}[1]{\left({#1}\right)}
\newcommand{\parenb}[1]{\left[{#1}\right]}

\newcommand{\RR}{\mathbb{R}}
\newcommand{\NN}{\mathbb{N}}

\newcommand{\sdiff}[2]{{#1} \backslash {#2}}

\newtoggle{full}

\iftoggle{cheat}{\togglefalse{full}}{}

\newcommand{\disp}{\iftoggle{full}{\displaystyle}{}}

\newcommand{\eeps}{\varepsilon}

\newcommand{\To}{\Rightarrow}

\usepackage{amsmath}
\usepackage{amssymb}
\usepackage{array}
\usepackage{enumitem}

\usepackage{multirow}
\usepackage{colortbl}
\usepackage{array}
\usepackage{wrapfig}
\usepackage{tabu}
\usepackage{makecell}
\usepackage{tabularx}

\newcolumntype{L}[1]{>{\raggedright\arraybackslash}p{#1}}
\newcolumntype{C}[1]{>{\centering\arraybackslash}p{#1}}
\newcolumntype{R}[1]{>{\raggedleft\arraybackslash}p{#1}}

\newcommand{\tsfrac}[2]{\textstyle \frac{{#1}}{{#2}}}

\newcommand{\inv}[1]{#1^{-1}}

\newcommand{\rint}[4]{\disp \int_{#1}^{#2} {#3}\ \mathrm{d}{#4}}

\newcommand{\NDist}{\mathcal{N}}

\newcommand{\Fol}{\sim}

\newcommand{\seq}[1]{({#1})}

\usepackage{bm}

\usepackage{pgf,tikz}
\usepackage{mathrsfs}
\usepackage{multicol}
\usepackage{bigints}

\usepackage{graphicx}

\makeatletter
\newcommand\backnot{\mathrel{\mathpalette\back@not\relax}}
\newcommand\back@not[2]{%
  \makebox[0pt][l]{%
    \sbox\z@{$\m@th#1=$}%
    \kern\wd\z@
    \reflectbox{$\m@th#1\not$}%
  }%
}
\makeatother

\newcommand{\sett}[1]{\left\{{#1}\right\}}

\makeatletter
\renewcommand*\env@matrix[1][*\c@MaxMatrixCols c]{%
  \hskip -\arraycolsep
  \let\@ifnextchar\new@ifnextchar
  \array{#1}}
\makeatother

\newcommand{\VaR}[2]{v_{#1}(#2)}
\newcommand{\CVaR}[2]{c_{#1}(#2)}
\newcommand{\sCVaR}[2]{\hat{c}_{#1}(#2)}
\newcommand{\PP}{\mathbb{P}}
\newcommand{\EE}{\mathbb{E}}

\newcommand{\nonneg}[1]{{#1}^+}
\newcommand{\argmin}{\mathrm{arg\ min}}

\newcommand{\Update}{\mathrm{Update}}
\newcommand{\kay}{\mathcal{K}}
\newcommand{\arr}{\mathcal{R}}
\newcommand{\subgap}[1]{\Delta(#1)}
\newcommand{\infgap}[2]{\Delta_{#1}(#2)}
\newcommand{\riskgap}[1]{\Delta_{\mathrm{r}}(#1,\alpha)}
\newcommand{\subreg}[2]{\arr_{#1}^{\mathsf{sub}}(#2)}
\newcommand{\infreg}[2]{\arr_{#1}^{\mathsf{inf}}(#2)}
\newcommand{\riskreg}[2]{\arr_{#1}^{\mathsf{risk}}(#2)}

%% file: raw/raw_supp/flags.tex
\begin{table}[t]
\begin{center}
\begin{tabular}{|c||c|c|c|}
\hline
	Setup & RC-LCB & CVaR-TS & \textsc{MaRaB}\\\hline
	$\sigma_>^2$ & $0.00$ & $0.00$ & $0.00$\\
	\hline
	$\sigma_\approx^2$ & $0.00$ & $0.00$ & $0.00$\\\hline
\end{tabular}
\caption{\label{tab:wrong_flagging} Comparison of the proportions of wrong flags $r/N$, i.e., flagging a feasible instance as infeasible. All algorithms perform well in the feasible instance where $\tau = 4.6$.}
\end{center}
\end{table}
\vspace{-.03in}
\begin{table}[t]
\begin{center}
\begin{tabular}{|c||c|c|c|}
\hline
	Setup & RC-LCB & CVaR-TS & \textsc{MaRaB}\\\hline
	$\sigma_>^2$ & $1.00$ & $0.02$ & $0.99$\\\hline
	$\sigma_\approx^2$ & $1.00$ & $0.03$ & $1.00$\\
	\hline
\end{tabular}
\caption{\label{tab:wrong_flagging2} Comparison of $r/N$ of flagging an infeasible instance as feasible.  When $\tau=2$,  to have $r/N\approx0$, we need $n\ge 312,678$ for RC-LCB. We omit this due to computational limitations.  }
\end{center}
\end{table}
\vspace{-.03in}
Another consideration 
is the flagging of instances as feasible when in fact they are infeasible and vice versa (Lines 19--23 of Algorithm~\ref{alg: CVaR-TS}). This was analyzed by \cite{kagrecha2020constrained}. \vt{In Tables~\ref{tab:wrong_flagging} and \ref{tab:wrong_flagging2}, we tabulate $r/N$ where $N=100$ is the total number of statistically independent runs and $r$ is the number of times that the various algorithms CVaR-TS (Algorithm~\ref{alg: CVaR-TS}), \textsc{MaRaB} \citep{galichet2013exploration}, and RC-LCB  \citep{kagrecha2020constrained} declare the flag incorrectly. Table~\ref{tab:wrong_flagging} compares the proportions of flagging a feasible instance as infeasible while Table~\ref{tab:wrong_flagging2} compares the proportions of flagging an infeasible instance as feasible. We see from Table~\ref{tab:wrong_flagging} that all algorithms perform well in this particular respect. However, from Table~\ref{tab:wrong_flagging2}, the situation is clearly in favor of CVaR-TS. This is partly because the parameters for the confidence bounds in RC-LCB and \textsc{MaRaB} are too conservative in favor of declaring feasibility. Furthermore, for RC-LCB, one requires a lower bound on  the horizon $n$ for the performance guarantee to hold  (see  Subsection~A.6 which contains the proof of Theorem~2 of \citet{kagrecha2020constrained}) and for the particular examples we experimented with, the lower bound was prohibitively large. Indeed, for RC-LCB, we must have $\hat{c}_\alpha(i) - \CVaR{\alpha}{i} > \tau - \CVaR{\alpha}{i}  + \frac{1}{1-\alpha}\sqrt{(\log(2D_\sigma n^2))/(nd_\sigma)}$ (and $\tau - \CVaR{\alpha}{i}<0$) which translates to   the horizon $n \geq 312, 678$. In contrast, CVaR-TS is parameter-free for attaining a far lower rate of incorrect flagging and outperforms RC-LCB and \textsc{MaRaB} in terms of ensuring that the error rate for declaring an infeasible instances as feasible and vice versa are both   small. }
	

%% file: raw/raw_supp/risk_proof_3.tex
\begin{center}
\textbf{\large Supplementary Material}
\end{center}


%
\section{Proof of Theorem~\ref{thm: risk_regret}}
\label{sec: S-1}
\begin{proof}[Proof of Theorem~\ref{thm: risk_regret}] Denote the sample CVaR at $\alpha$ by $\sCVaR{\alpha}{i,t} = \theta_{i,t}\paren{\frac{\alpha}{1-\alpha}} + \frac{1}{\sqrt{\kappa_{i,t}}} c_\alpha^*$ (see (\ref{eqn: CVaR Gaussian})). Fix $\eeps > 0$, and define
\begin{equation*}
	E_i(t) := \sett{\sCVaR{\alpha}{i,t} > \CVaR{\alpha}{1} + \eeps},
\end{equation*}
the event that the Thompson sample CVaR of arm $i$ is $\eeps$-higher than the optimal arm (which has the lowest CVaR). Intuitively, event $E_i(t)$ is highly likely to occur when the algorithm has explored sufficiently. However, the algorithm does not choose arm $i$ when $E_i^c(t)$, an event with small probability under Thompson sampling, occurs. By Lemma~\ref{subthm: lattimore_main}, and the linearity of expectation, we can divide $\EE[T_{i,n}]$ into two parts as 
\begin{equation}
\label{eqn: key_lemma}
\EE[T_{i,n}] \leq \sum_{s=0}^{n-1} \EE \parenb{\frac{1}{G_{1s}} - 1} +  \sum_{s=0}^{n-1} \PP \paren{G_{is} > \frac{1}{n}} + 1.\end{equation}
By Lemmas~\ref{lem: upper_bound_term_1_risk} and \ref{lem: upper_bound_term_2_risk} in the following, we have
\begin{align*}
\sum_{s=1}^{n} \EE \parenb{\frac{1}{G_{1s}} - 1} &\leq \frac{C_1}{\eeps^3} + \frac{C_2}{\eeps^2} + \frac{C_3}{\eeps} + C_4,\ \text{and}\\
\sum_{s=1}^n \PP_t\paren{G_{is} > \frac{1}{n}}
&\leq 1 + \max \sett{\frac{2\alpha^2\log (2n)}{\xi^2{(1-\alpha)}^2\paren{\riskgap{i} - \eeps}^2}, \frac{\log (2n)}{h \paren{\frac{\sigma_i^2 {(c_\alpha^*)}^2}{\paren{\sigma_i  c_\alpha^*- (1 - \xi)\paren{\riskgap{i} - \eeps}}^2}}}} + \frac{C_5}{\eeps^4} + \frac{C_6}{\eeps^2}.
\end{align*}
Plugging the two displays into (\ref{eqn: key_lemma}), we have
\begin{equation}
	\EE[T_{i,n}] \leq 1 + \max \sett{\frac{2 \alpha^2 \log (2n)}{\xi^2{(1-\alpha)}^2\paren{\riskgap{i} - \eeps}^2}, \frac{\log (2n)}{h \paren{\frac{\sigma_i^2 {(c_\alpha^*)}^2}{\paren{\sigma_i (c_\alpha^*) - (1 - \xi)\paren{\riskgap{i} - \eeps}}^2}}}} + \frac{C_1'}{\eeps^4} + \frac{C_2'}{\eeps^3} + \frac{C_3'}{\eeps^2} + \frac{C_4'}{\eeps} + C_5', \label{eqn: pre_result}
\end{equation}
where $C_1',C_2',C_3',C_4',C_5'$ are constants. Setting $\eeps = {(\log n)}^{-\frac{1}{8}}$ into (\ref{eqn: pre_result}), we get
\begin{align*}
\limsup_{n\to\infty}\frac{\riskreg{n}{\text{CVaR-TS}}}{\log n}
&\leq \sum_{i \in \sdiff{[K]}{\kay^*}} \paren{\max \sett{\frac{2\alpha^2 }{\xi^2{(1-\alpha)}^2\Delta_r^2(i)}, \frac{1}{h \paren{\frac{\sigma_i^2 {(c_\alpha^*)}^2}{{\paren{\sigma_i c_\alpha^* - (1 - \xi)\paren{\riskgap{i}}}}^2}}}}}\riskgap{i}\\
&=\sum_{i \in \sdiff{[K]}{\kay^*}} \max\sett{A_{\alpha,\xi}^i,B_{\alpha,\xi}^i} \riskgap{i} =\sum_{i \in \sdiff{[K]}{\kay^*}} C_{\alpha,\xi}^i\riskgap{i}.
\end{align*}
\end{proof}

\begin{lemma}
\label{lem: tail_lower_bd} We can lower bound
\begin{equation*}
 	\PP_t\paren{E_1^c(t) \mid T_{1,t} = s, \hat{\mu}_{1,s} = \mu, \hat{\sigma}_{1,s} = \sigma}=\PP_t\paren{\sCVaR{\alpha}{1,t} \leq \CVaR{\alpha}{1} + \eeps \mid T_{1,t} = s, \hat{\mu}_{1,s} = \mu, \hat{\sigma}_{1,s} = \sigma}
 \end{equation*}
by
\begin{align}
&\PP_t\paren{\sCVaR{\alpha}{1,t} \leq \CVaR{\alpha}{1} + \eeps \mid T_{1,t} = s, \hat{\mu}_{1,s} = \mu, \hat{\sigma}_{1,s} = \sigma} \nonumber\\
&\geq \begin{cases}
 	\PP_t\paren{{\theta}_{1,t} - \mu_1\leq \frac{(1-\alpha)\eeps}{2\alpha}} \cdot \PP_t \paren{\frac{1}{\sqrt{\kappa_{1,t}}}-\sigma_1 \leq \frac{\eeps}{2c_\alpha^*}} &\text{if } \mu \geq \mu_1, \sigma \geq \sigma_1,\\
 	\frac{1}{2} \PP_t \paren{\frac{1}{\sqrt{\kappa_{1,t}}}-\sigma_1 \leq \frac{\eeps}{2c_\alpha^*}} &\text{if } \mu < \mu_1, \sigma \geq \sigma_1,\\
 	\frac{1}{2} \PP_t\paren{{\theta}_{1,t} - \mu_1\leq \frac{(1-\alpha)\eeps}{2\alpha}} &\text{if } \mu \geq \mu_1, \sigma < \sigma_1,\\
 	\frac{1}{4} &\text{if } \mu < \mu_1, \sigma < \sigma_1.
 \end{cases}\label{eqn: lower_bd}
\end{align}
\end{lemma}
\begin{proof}[Proof of Lemma~\ref{lem: tail_lower_bd}]
Given $T_{1,t} = s, \hat{\mu}_{1,s} = \mu, \hat{\sigma}_{1,s} = \sigma$, a direct calculation gives us,
\begin{align*}
&\PP_t\paren{\sCVaR{\alpha}{1,t} \leq \CVaR{\alpha}{1} + \eeps \mid T_{1,t} = s, \hat{\mu}_{1,s} = \mu, \hat{\sigma}_{1,s} = \sigma}\\
&= \PP_t\paren{({\theta}_{1,t}-\mu_1)\paren{\frac{\alpha}{1-\alpha}} + \paren{\frac{1}{\sqrt{\kappa_{1,t}}}-\sigma_1} c_\alpha^* \leq \eeps}\\
&\geq \PP_t\paren{{\theta}_{1,t} - \mu_1\leq \frac{(1-\alpha)\eeps}{2\alpha}} \cdot \PP_t \paren{\frac{1}{\sqrt{\kappa_{1,t}}}-\sigma_1 \leq \frac{\eeps}{2c_\alpha^*}}.
\end{align*}
If $\mu < \mu_1$, then $$\PP_t\paren{\theta_{1,t} - \mu_1 \leq \frac{(1-\alpha)\eeps}{2\alpha}} = \PP_t\paren{\theta_{1,t} - \mu \leq \mu_1 - \mu + \frac{(1-\alpha)\eeps}{2\alpha}} \geq \frac{1}{2}$$ by the properties of the median of the Gaussian distribution.

If $\sigma < \sigma_1$, then $$\PP_t \paren{\frac{1}{\sqrt{\kappa_{1,t}}} - \sigma_1 \leq \frac{\eeps}{2c_\alpha^*}} = \PP_t\paren{\frac{1}{\kappa_{1,t}} -\sigma_1^2 \leq \paren{\frac{\eeps}{2c_\alpha^*} + 2 \sigma_1}\frac{\eeps}{2c_\alpha^*}} \geq \frac{1}{2}$$ by the properties of the median of the Gamma distribution.
\end{proof}
\begin{lemma}[Upper bounding the first term of (\ref{eqn: key_lemma})]
\label{lem: upper_bound_term_1_risk}
	We have \begin{equation*}
		\sum_{s=1}^{n} \EE \parenb{\frac{1}{G_{1s}} - 1} \leq \frac{C_1}{\eeps^3} + \frac{C_2}{\eeps^2} + \frac{C_3}{\eeps} + C_4,
	\end{equation*}
	where $C_1,C_2,C_3,C_4$ are constants.
\end{lemma}
\begin{proof}[Proof of Lemma~\ref{lem: upper_bound_term_1_risk}]
We now attempt to bound $\EE \parenb{\frac{1}{G_{1s}} - 1}$ by conditioning on the various values of $\hat{\mu}_{i,s}$ and $\hat{\sigma}_{i,s}$.
Firstly, we define the conditional version of $G_{1s}$ by $$\widetilde{G}_{1s} = G_{1s} |_{\hat{\mu}_{1,s} = \mu, \hat{\sigma}_{1,s} = \beta}= \PP_t \paren{\sCVaR{\alpha}{1,t} \leq \CVaR{\alpha}{1} + \eeps \mid T_{i,t} = s, \hat{\mu}_{1,s} = \mu, \hat{\sigma}_{1,s} = \beta},$$ which is the $\mathrm{LHS}$ of (\ref{eqn: lower_bd}).
Define $c_1 = \frac{1}{\sqrt{2 \pi \sigma_1^2}}$, $c_2 = \frac{1}{2^{\frac{s-3}{2}}\Gamma(\frac{s-1}{2})\sigma_1^{s-1}}$, and $\omega = s\sqrt{2} \sigma_1$.
For clarity, we partition the parameter space $(\beta, \mu) \in [0,\infty) \times (-\infty,\infty) = A \cup B \cup C \cup D,$
where
$$A = [0,\omega) \times [\mu_1 + \tsfrac{(1-\alpha)\eeps}{2\alpha}, \infty), \;\; B = [0,\omega) \times (-\infty, \tsfrac{(1-\alpha)\eeps}{2\alpha}], \;\; C = [\omega,\infty) \times [\tsfrac{(1-\alpha)\eeps}{2\alpha}, \infty), \;\; D = [\omega,\infty) \times (-\infty, \tsfrac{(1-\alpha)\eeps}{2\alpha}].$$
We can then partition $\EE \parenb{\frac{1}{G_{1s}} - 1}$ into four parts:
\begin{align*}
\EE \parenb{\frac{1}{G_{1s}} - 1} &= c_1 c_2 \rint{0}{\infty}{\rint{-\infty}{\infty}{\frac{1-\widetilde{G}_{1s}}{\widetilde{G}_{1s}} \exp \paren{-\frac{s{(\mu-\mu_1)}^2}{2 \sigma_1^2}}\beta^{s-2} e^{-\frac{\beta^2}{2 \sigma_1^2}}}{\mu}}{\beta}\\
&= c_1 c_2 \paren{\int_A + \int_B + \int_C + \int_D}\frac{1-\widetilde{G}_{1s}}{\widetilde{G}_{1s}} \exp \paren{-\frac{s{(\mu-\mu_1)}^2}{2 \sigma_1^2}}\beta^{s-2} e^{-\frac{\beta^2}{2 \sigma_1^2}} \mathrm{d}\mu \mathrm{d}\beta.\end{align*}

For Part $B$, using the fourth case in Lemma~\ref{lem: tail_lower_bd}, $$\frac{1 - \widetilde{G}_{1s}}{\widetilde{G}_{1s}} \leq 4(1 - \widetilde{G}_{1s}).$$
It follows that
\begin{align}
&c_1 c_2 \int_B \frac{1-\widetilde{G}_{1s}}{\widetilde{G}_{1s}} \exp \paren{-\frac{s{(\mu-\mu_1)}^2}{2 \sigma_1^2}}\beta^{s-2} e^{-\frac{\beta^2}{2 \sigma_1^2}}  \, \mathrm{d}\mu \, \mathrm{d}\beta \nonumber \\
&\leq 4c_1 c_2 \int_B (1 - \widetilde{G}_{1s})\exp \paren{-\frac{s{(\mu-\mu_1)}^2}{2 \sigma_1^2}}\beta^{s-2} e^{-\frac{\beta^2}{2 \sigma_1^2}}  \,\mathrm{d}\mu \, \mathrm{d}\beta \nonumber\\
&\leq 4c_1 c_2\int_B \paren{\PP_t\paren{\theta_{1,t} - \mu_1 \geq \frac{(1-\alpha)\eeps}{2\alpha} \Big| \hat{\mu}_{1,s} = \mu} + \PP \paren{\frac{1}{\sqrt{\kappa_{1,t}}} - \sigma_1 \geq \frac{\eeps}{2c_\alpha^*} \Big| \hat{\sigma}_{1,s} = \beta}}\nonumber\\
& \phantom{--} \cdot \exp \paren{-\frac{s{(\mu-\mu_1)}^2}{2 \sigma_1^2}}\beta^{s-2} e^{-\frac{\beta^2}{2 \sigma_1^2}}  \, \mathrm{d}\mu  \, \mathrm{d}\beta \nonumber\\
& \leq 4 c_1 \int_{-\infty}^{\mu_1 + \frac{(1-\alpha)\eeps}{2\alpha}} \PP_t\paren{\theta_{1,t} - \mu_1 \geq \frac{(1-\alpha)\eeps}{2\alpha} \Big| \hat{\mu}_{1,s} = \mu} \exp \paren{-\frac{s{(\mu-\mu_1)}^2}{2 \sigma_1^2}}\ \mathrm{d}\mu \nonumber\\
& \phantom{--} + 4 c_2 \int_0^{\omega} \PP\paren{\frac{1}{\sqrt{\kappa_{1,t}}} - \sigma_1 \geq \frac{\eeps}{2c_\alpha^*} \Big| \hat{\sigma}_{1,s} = \beta}\beta^{s-2} e^{-\frac{\beta^2}{2 \sigma_1^2}}\ \mathrm{d}\beta \nonumber\\
& \leq 4 c_1 \int_{-\infty}^{\mu_1 + \frac{(1-\alpha)\eeps}{2\alpha}} \exp \paren{-\frac{s}{2}{\paren{\mu_1 - \mu + \frac{(1-\alpha)\eeps}{2\alpha}}}^2} \exp \paren{-\frac{s{(\mu-\mu_1)}^2}{2 \sigma_1^2}}\ \mathrm{d}\mu \nonumber\\
& \phantom{--} + 4 c_2 \int_0^{\omega} \PP\paren{\kappa_{1,t} \leq \frac{1}{{\paren{\frac{\eeps}{2c_\alpha^*} + \sigma_1}}^2} \Big| \hat{\sigma}_{1,s} = \beta}\beta^{s-2} e^{-\frac{\beta^2}{2 \sigma_1^2}}\ \mathrm{d}\beta \label{step: Part_B2}\\
& \leq 4 c_1 \int_{-\infty}^{\mu_1 + \frac{(1-\alpha)\eeps}{2\alpha}} \exp \paren{-\frac{s}{2}{\paren{\mu_1 - \mu + \frac{(1-\alpha)\eeps}{2\alpha}}}^2} \exp \paren{-\frac{s{(\mu-\mu_1)}^2}{2 \sigma_1^2}}\ \mathrm{d}\mu \nonumber\\
& \phantom{--} + 4 c_2 \int_0^{\omega} \exp\paren{-\frac{\paren{\beta^2 - s\paren{\frac{\eeps}{2c_\alpha^*} + \sigma_1}^2}^2}{4s \paren{\frac{\eeps}{2c_\alpha^*} + \sigma_1}^4}}\beta^{s-2} e^{-\frac{\beta^2}{2 \sigma_1^2}}\ \mathrm{d}\beta \label{step: Part_B}\\
&\leq 4 \exp \paren{-\frac{s {(1-\alpha)}^2 \eeps^2}{16 \alpha^2}} + 4 \exp \paren{-\frac{s\eeps^2}{16 {(c_\alpha^*)}^2}}.\nonumber
\end{align}
\normalsize
where \eqref{step: Part_B2} and   (\ref{step: Part_B}) respectively follow from applying standard tail upper bounds on the Gaussian and Gamma distributions.  


For Part $A$, using the third case in the above lemma, $$\frac{1 - \widetilde{G}_{1s}}{\widetilde{G}_{1s}} \leq  \frac{2}{\PP_t \paren{\theta_{1,t} - \mu_1 \leq \frac{(1-\alpha)\eeps}{2\alpha} | \hat{\mu}_{1,s} = \mu}}.$$
Then
\small \begin{align}
&c_1 c_2 \int_A \frac{1-\widetilde{G}_{1s}}{\widetilde{G}_{1s}} \exp \paren{-\frac{s{(\mu-\mu_1)}^2}{2 \sigma_1^2}}\beta^{s-2} e^{-\frac{\beta^2}{2 \sigma_1^2}} \, \mathrm{d}\mu  \,\mathrm{d}\beta \nonumber\\
&\leq 2c_1 c_2 \int_A \frac{1}{\PP_t \paren{\theta_{1,t} - \mu_1 \leq \frac{(1-\alpha)\eeps}{2\alpha} | \hat{\mu}_{1,s} = \mu}} \exp \paren{-\frac{s{(\mu-\mu_1)}^2}{2 \sigma_1^2}}\beta^{s-2} e^{-\frac{\beta^2}{2 \sigma_1^2}}  \,\mathrm{d}\mu \, \mathrm{d}\beta \nonumber\\
&\leq 2c_1 \int_{\mu_1+\frac{(1-\alpha)\eeps}{2\alpha}}^{\infty} \frac{1}{\PP_t \paren{\theta_{1,t} - \mu_1 \leq \frac{(1-\alpha)\eeps}{2\alpha} | \hat{\mu}_{1,s} = \mu}} \exp \paren{-\frac{s{(\mu-\mu_1)}^2}{2 \sigma_1^2}} \mathrm{d}\mu \cdot c_2 \int_0^{\omega} \beta^{s-2} e^{-\frac{\beta^2}{2 \sigma_1^2}}  \,\mathrm{d}\beta \nonumber\\
&\leq \frac{2c_1 \sqrt{2}}{\sqrt{\pi}} \int_{\mu_1+\frac{(1-\alpha)\eeps}{2\alpha}}^{\infty}\paren{\sqrt{s} \paren{\mu -\mu_1 - \frac{(1-\alpha)\eeps}{2\alpha}} + \sqrt{s{\paren{\mu -\mu_1 - \frac{(1-\alpha)\eeps}{2\alpha}}}^2 + 4}} \nonumber\\
&\phantom{--} \cdot \exp \paren{\frac{s{\paren{\mu -\mu_1 - \frac{(1-\alpha)\eeps}{2\alpha}}}^2}{2}}\cdot \exp \paren{-\frac{s{(\mu-\mu_1)}^2}{2 \sigma_1^2}}  \,\mathrm{d}\mu \label{step: Part_A}\\
&\leq \frac{2c_1 \sqrt{2}}{\sqrt{\pi}} \int_{0}^{\infty}\paren{\sqrt{s} z + \sqrt{sz^2 + 4}} \exp \paren{\frac{sz^2}{2}}\cdot \exp \paren{-\frac{s\paren{z+\frac{(1-\alpha)\eeps}{2\alpha}}^2}{2 }} \, \mathrm{d}z \nonumber\\
&\leq \frac{2c_1 \sqrt{2}}{\sqrt{\pi}} \exp \paren{-\frac{s\paren{\frac{(1-\alpha)\eeps}{2\alpha}}^2}{2}} \int_{0}^{\infty}3\paren{\sqrt{s} z}\cdot \exp \paren{ - sz\eeps} \mathrm{d}z \leq \frac{6}{\pi \eeps^2 s \sqrt{s}} \exp \paren{-\frac{s {(1-\alpha)}^2 \eeps^2}{8\alpha^2}}, \nonumber
\end{align} \normalsize
where (\ref{step: Part_A}) follows from Lemma~\ref{lem: abram}, which we state below.
\begin{lemma}[\citet{abramowitz+stegun}]
\label{lem: abram}
For $X \Fol \mathcal{N}(\mu, 1/s)$, $$\PP(X \leq \mu - x) = \PP(X \geq \mu + x) \geq \sqrt{\frac{2}{\pi}} \cdot \frac{\exp \paren{-\frac{sx^2}{2}}}{\sqrt{s} x + \sqrt{sx^2 + 4}}.$$
\end{lemma}
 

For Part $D$, using the second case in Lemma~\ref{lem: tail_lower_bd}, $$\frac{1 - \widetilde{G}_{1s}}{\widetilde{G}_{1s}} \leq  \frac{2}{\PP_t \paren{\frac{1}{\sqrt{\kappa_{1,t}}} - \sigma_1 \leq \frac{\eeps}{2c_\alpha^*} \big| \hat{\sigma}_{1,s} = \beta}}.$$
Then
\begin{align}
&c_1 c_2 \int_D \frac{1-\widetilde{G}_{1s}}{\widetilde{G}_{1s}} \exp \paren{-\frac{s{(\mu-\mu_1)}^2}{2 \sigma_1^2}}\beta^{s-2} e^{-\frac{\beta^2}{2 \sigma_1^2}}  \,\mathrm{d}\mu  \,\mathrm{d}\beta\nonumber\\
&\leq 2c_1 c_2 \int_D \frac{1}{\PP_t \paren{\frac{1}{\sqrt{\kappa_{1,t}}} - \sigma_1 \leq \frac{\eeps}{2c_\alpha^*} \big| \hat{\sigma}_{1,s} = \beta}}\exp \paren{-\frac{s{(\mu-\mu_1)}^2}{2 \sigma_1^2}} \beta^{s-2} e^{-\frac{\beta^2}{2 \sigma_1^2}} \, \mathrm{d}\mu  \,\mathrm{d}\beta\nonumber\\
&\leq c_1 \int_{-\infty}^{\mu_1}  \exp \paren{-\frac{s{(\mu-\mu_1)}^2}{2 \sigma_1^2}} \mathrm{d}\mu \cdot 2c_2 \int_{\omega}^{\infty} \frac{1}{\PP_t \paren{\frac{1}{\sqrt{\kappa_{1,t}}} - \sigma_1 \leq \frac{\eeps}{2c_\alpha^*} \big| \hat{\sigma}_{1,s} = \beta}} \beta^{s-2} e^{-\frac{\beta^2}{2 \sigma_1^2}}  \,\mathrm{d}\beta\nonumber\\
&\leq 2c_2 \int_{\omega}^{\infty} \frac{1}{\PP_t \paren{\frac{1}{\sqrt{\kappa_{1,t}}} - \sigma_1 \leq \frac{\eeps}{2c_\alpha^*} \big| \hat{\sigma}_{1,s} = \beta}} \beta^{s-2} e^{-\frac{\beta^2}{2 \sigma_1^2}}  \,\mathrm{d}\beta\nonumber\\
&\leq 2c_2 \Gamma \paren{\frac{s}{2}} \int_{\omega}^{\infty} \exp \paren{\frac{\beta^2}{ {(\frac{\eeps}{2c_\alpha^*} + \sigma_1)}^2}-\frac{\beta^2}{2 \sigma_1^2}} \beta^{s-2} \paren{1 + \frac{\beta^2}{{(\frac{\eeps}{2c_\alpha^*} + \sigma_1)}^2}}^{-\paren{\frac{s}{2}-1}} \, \mathrm{d}\beta \label{step: zhu}\\
&\leq 2c_2 \Gamma \paren{\frac{s}{2}} \cdot \paren{\frac{\eeps}{2c_\alpha^*}+\sigma_1}^{s-1} \int_{\frac{\sqrt{2}\sigma_1 s}{\frac{\eeps}{2c_\alpha^*} + \sigma_1}}^{\infty} \exp \paren{y^2 - \frac{y^2\paren{\frac{\eeps}{2c_\alpha^*} + \sigma_1}^2}{2\sigma_1^2}} y^{s-2} \paren{1 + y^2}^{-\paren{\frac{s}{2}-1}} \, \mathrm{d}y\nonumber\\
&\leq 2c_2 \Gamma \paren{\frac{s}{2}} \cdot \paren{\frac{\eeps}{2c_\alpha^*}+\sigma_1}^{s-1}  \int_{s}^{\infty} y \exp \paren{-
\frac{2\sigma_1^2 - \paren{\frac{\eeps}{2c_\alpha^*} + \sigma_1}^2}{\sigma_1^2}
y^2}  \,\mathrm{d}y\nonumber\\
&\leq \frac{\sigma_1^2c_2 \Gamma \paren{\frac{s}{2}} \cdot \paren{\frac{\eeps}{2c_\alpha^*}+\sigma_1}^{s-1}}{ 2\sigma_1^2 - \paren{\frac{\eeps}{2c_\alpha^*} + \sigma_1}^2}\exp \paren{-\frac{2\sigma_1^2 - \paren{\frac{\eeps}{2c_\alpha^*} + \sigma_1}^2}{\sigma_1^2}s^2}. \nonumber
\end{align}
where (\ref{step: zhu}) follows from Lemma~\ref{lem: zhu}, which we state below.
\begin{lemma}[{\citet[Lemma S-1]{zhu2020thompson}}]
\label{lem: zhu}
If $X \Fol \mathrm{Gamma}(\alpha,\beta)$ with $\alpha \geq 1$ and rate $\beta > 0$, we have the lower bound on the complementary CDF $$\PP(X \geq x) \geq \frac{1}{\Gamma(\alpha)} \exp(-\beta x) {(1 + \beta x)}^{\alpha-1}.$$
\end{lemma}
 

For Part $C$, using the first case in Lemma~\ref{lem: tail_lower_bd},
$$\frac{1 - \widetilde{G}_{1s}}{\widetilde{G}_{1s}} \leq \frac{1}{\PP_t\paren{{\theta}_{1,t} - \mu_1\leq \eeps| \hat{\mu}_{1,s} = \mu} \cdot \PP_t \paren{\frac{1}{\sqrt{\kappa_{1,t}}}-\sigma_1 \leq \eeps| \hat{\sigma}_{1,s} = \beta}}.$$
Reusing the integrations in Parts $A$ and $D$,
\begin{align*}
&c_1 c_2 \int_D \frac{1-\widetilde{G}_{1s}}{\widetilde{G}_{1s}} \exp \paren{-\frac{s{(\mu-\mu_1)}^2}{2 \sigma_1^2}}\beta^{s-2} e^{-\frac{\beta^2}{2 \sigma_1^2}}  \,\mathrm{d}\mu  \,\mathrm{d}\beta\\
&\leq c_1 c_2 \int_D \frac{1}{\PP_t\paren{{\theta}_{1,t} - \mu_1\leq \frac{(1-\alpha)\eeps}{2\alpha} \big| \hat{\mu}_{1,s} = \mu} \cdot \PP_t \paren{\frac{1}{\sqrt{\kappa_{1,t}}}-\sigma_1 \leq \frac{\eeps}{2c_\alpha^*} \big| \hat{\sigma}_{1,s} = \beta}}\\
&\phantom{---}\cdot \exp \paren{-\frac{s{(\mu-\mu_1)}^2}{2 \sigma_1^2}}\beta^{s-2} e^{-\frac{\beta^2}{2 \sigma_1^2}} \, \mathrm{d}\mu  \,\mathrm{d}\beta\\
&\leq \frac{1}{4} \cdot 2c_1 \int_{\mu_1+\frac{(1-\alpha)\eeps}{2\alpha}}^{\infty} \frac{1}{\PP_t(\theta_{1,t} - \mu_1 \leq \frac{(1-\alpha)\eeps}{2\alpha} \big| \hat{\mu}_{1,s} = \mu)} \exp \paren{-\frac{s{(\mu-\mu_1)}^2}{2 \sigma_1^2}} \, \mathrm{d}\mu\\
&\phantom{---}\cdot 2c_2 \int_{\omega}^{\infty} \frac{1}{\PP_t (\frac{1}{\sqrt{\kappa_{1,t}}} - \sigma_1 \leq \frac{\eeps}{2c_\alpha^*} \big| \hat{\sigma}_{1,s} = \beta)} \beta^{s-2} e^{-\frac{\beta^2}{2 \sigma_1^2}}  \,\mathrm{d}\beta\nonumber\\
&\leq \frac{3}{\pi \eeps^2 s \sqrt{s}} \exp \paren{-\frac{s {(1-\alpha)}^2 \eeps^2}{8\alpha^2}} \cdot \frac{\sigma_1^2c_2 \Gamma \paren{\frac{s}{2}} \cdot \paren{\frac{\eeps}{2c_\alpha^*}+\sigma_1}^{s-1}}{ 2\sigma_1^2 - \paren{\frac{\eeps}{2c_\alpha^*} + \sigma_1}^2}\exp \paren{-\frac{2\sigma_1^2 - \paren{\frac{\eeps}{2c_\alpha^*} + \sigma_1}^2}{\sigma_1^2}s^2}.
\end{align*}
Combining these four parts, we can upper bound $\EE \parenb{\frac{1}{G_{1s}} - 1}$ by 
\begin{align*}
\EE \parenb{\frac{1}{G_{1s}} - 1} &= 4 \exp \paren{-\frac{s {(1-\alpha)}^2 \eeps^2}{16 \alpha^2}} + 4 \exp \paren{-\frac{s\eeps^2}{16 {(c_\alpha^*)}^2}} + \frac{6}{\pi \eeps^2 s \sqrt{s}} \exp \paren{-\frac{s {(1-\alpha)}^2 \eeps^2}{8\alpha^2}} \\
&\phantom{---}+ \frac{2 \sigma_1 \paren{\frac{\eeps}{2c_\alpha^*} + \sigma_1}^{s-1}}{\frac{\eeps}{2c_\alpha^*}}\exp \paren{- \frac{\frac{\eeps}{2c_\alpha^*}(\frac{\eeps}{2c_\alpha^*}+2\sigma_1)}{\sigma_1^2}s^2}\\
&\phantom{---}+ \frac{3}{\pi \eeps^2 s \sqrt{s}} \exp \paren{-\frac{s {(1-\alpha)}^2 \eeps^2}{8\alpha^2}} \cdot \frac{\sigma_1^2c_2 \Gamma \paren{\frac{s}{2}} \cdot \paren{\frac{\eeps}{2c_\alpha^*}+\sigma_1}^{s-1}}{ 2\sigma_1^2 - \paren{\frac{\eeps}{2c_\alpha^*} + \sigma_1}^2}\exp \paren{-\frac{2\sigma_1^2 - \paren{\frac{\eeps}{2c_\alpha^*} + \sigma_1}^2}{\sigma_1^2}s^2}.
\end{align*}
Summing over $s$, we have
\begin{align*}
\sum_{s=1}^{n} \EE \parenb{\frac{1}{G_{1s}} - 1} \leq \frac{C_1}{\eeps^3} + \frac{C_2}{\eeps^2} + \frac{C_3}{\eeps} + C_4.
\end{align*}
\end{proof}

 
\begin{lemma}
\label{lem: tail_upper_bd}
For $\xi \in (0,1)$, we have
\begin{align*}
&\PP_t\paren{\sCVaR{\alpha}{i,t} \leq \CVaR{\alpha}{1} + \eeps \mid T_{i,s} = s, \hat{\mu}_{i,s} = \mu, \hat{\sigma}_{i,s} = \sigma} \nonumber\\
&\leq \exp\paren{-\frac{s}{2} \paren{\mu - \mu_i +\frac{\xi(\riskgap{i} - \eeps)(1-\alpha)}{\alpha}}^2}+ \exp \paren{-sh \paren{\frac{\sigma^2 {(c_\alpha^*)}^2}{\paren{\sigma_i c_\alpha^* - (1 - \xi)\paren{\riskgap{i} - \eeps}}^2}}},
\end{align*}
where $h(x) = \frac{1}{2}(x-1-\log x)$.
\end{lemma}
\begin{proof}[Proof of Lemma~\ref{lem: tail_upper_bd}] Given $T_{i,t} = s, \hat{\mu}_{i,s} = \mu, \hat{\sigma}_{i,s} = \sigma$, a direct calculation gives us,
\begin{align} 
&\PP_t\paren{\sCVaR{\alpha}{i,t} \leq \CVaR{\alpha}{1} + \eeps \mid T_{i,t} = s, \hat{\mu}_{i,s} = \mu, \hat{\sigma}_{i,s} = \sigma} \nonumber\\
&= \PP_t\paren{{\theta}_{i,t}\paren{\frac{\alpha}{1-\alpha}} + \frac{1}{\sqrt{\kappa_{i,t}}} c_\alpha^* \leq \CVaR{\alpha}{1} +\eeps} \nonumber\\
&= \PP_t\paren{({\theta}_{i,t}-\mu_i)\paren{\frac{\alpha}{1-\alpha}} + \paren{\frac{1}{\sqrt{\kappa_{i,t}}}-\sigma_i} c_\alpha^* \leq - \riskgap{i} + \eeps} \nonumber\\
&= \PP_t\paren{({\theta}_{i,t}-\mu_i)\paren{\frac{\alpha}{1-\alpha}} + \paren{\frac{1}{\sqrt{\kappa_{i,t}}}-\sigma_i} c_\alpha^* \leq \paren{-\xi + (-1+\xi)} \paren{\riskgap{i} - \eeps}} \nonumber\\
&\leq \PP_t\paren{{\theta}_{i,t} - \mu\leq -\paren{\mu - \mu_i +\frac{\xi(\riskgap{i} - \eeps)(1-\alpha)}{\alpha}}} + \PP_t \paren{\frac{1}{\sqrt{\kappa_{i,t}}}-\sigma_i \leq -\frac{(1 - \xi)\paren{\riskgap{i} - \eeps}}{c_\alpha^*}} \nonumber\\
&\leq \exp\paren{-\frac{s}{2}  \paren{\mu - \mu_i +\frac{\xi(\riskgap{i} - \eeps)(1-\alpha)}{\alpha}}^2} + \PP_t \paren{\kappa_{i,t} \geq \frac{{(c_\alpha^*)}^2}{\paren{\sigma_i c_\alpha^* - (1 - \xi)\paren{\riskgap{i} - \eeps}}^2}} \nonumber\\
&\leq \exp\paren{-\frac{s}{2} \paren{\mu - \mu_i +\frac{\xi(\riskgap{i} - \eeps)(1-\alpha)}{\alpha}}^2} + \exp \paren{-sh \paren{\frac{\sigma^2 {(c_\alpha^*)}^2}{\paren{\sigma_i c_\alpha^* - (1 - \xi)\paren{\riskgap{i} - \eeps}}^2}}}, \label{step: harre}
\end{align}
where (\ref{step: harre}) follows from Lemma~\ref{lem: harre}, which we state below.
\begin{lemma}[\citet{Harremo_s_2017}]
\label{lem: harre}
For a Gamma r.v. $X \Fol \mathrm{Gamma}(\alpha,\beta)$ with shape $\alpha \geq 2$ and rate $\beta > 0$, we have $$\PP(X \geq x) \leq \exp \paren{-2\alpha h \paren{\frac{\beta x}{\alpha}}},\ x > \frac{\alpha}{\beta},$$ where $h(x) = \frac{1}{2}(x-1-\log x)$.
\end{lemma}
\end{proof}
\begin{lemma}[Upper bounding the second term of (\ref{eqn: key_lemma})]
\label{lem: upper_bound_term_2_risk}
	We have
\begin{align}
\sum_{s=1}^n \PP_t\paren{G_{is} > \frac{1}{n}}
&\leq 1 + \max \sett{\frac{2\alpha^2\log (2n)}{\xi^2{(1-\alpha)}^2\paren{\riskgap{i} - \eeps}^2}, \frac{\log (2n)}{h \paren{\frac{\sigma_i^2 {(c_\alpha^*)}^2}{\paren{\sigma_i  c_\alpha^*- (1 - \xi)\paren{\riskgap{i} - \eeps}}^2}}}} + \frac{C_5}{\eeps^4} + \frac{C_6}{\eeps^2},
\end{align}
where $C_5,C_6$ are constants.
\end{lemma}

\begin{proof}[Proof of Lemma~\ref{lem: upper_bound_term_2_risk}] From Lemma~\ref{lem: tail_upper_bd} we have the inclusions
\begin{align*}
\sett{\hat{\mu}_{i,s} - \sqrt{\frac{2 \log (2n)}{s}} \geq \mu_i - \frac{\xi(\riskgap{i} - \eeps)(1-\alpha)}{\alpha}} &\subseteq \sett{\exp\paren{-\frac{s}{2}\paren{\mu - \mu_i -\frac{\xi(\riskgap{i} - \eeps)(1-\alpha)}{\alpha}}^2} \leq \frac{1}{2n}}	
\end{align*}
and
\begin{align*}
\sett{\frac{\hat{\sigma}_{i,s}^2 {(c_\alpha^*)}^2}{\paren{\sigma_i c_\alpha^* - (1 - \xi)\paren{\riskgap{i} - \eeps}}^2} \geq \inv{h_+} \paren{\frac{\log (2n)}{s}}} &\cup \sett{\frac{\hat{\sigma}_{i,s}^2 {(c_\alpha^*)}^2}{\paren{\sigma_i c_\alpha^* - (1 - \xi)\paren{\riskgap{i} - \eeps}}^2} \leq \inv{h_-} \paren{\frac{\log (2n)}{s}}}\\
&\subseteq \sett{\exp \paren{-sh \paren{\frac{\sigma^2 {(c_\alpha^*)}^2}{\paren{\sigma_i c_\alpha^* - (1 - \xi)\paren{\riskgap{i} - \eeps}}^2}}}\leq \frac{1}{2n}}
\end{align*}
where $\inv{h_+}(y) = \max \sett{x : h(x) = y}$ and $\inv{h_-}(y) = \min \sett{x : h(x) = y}$.
Hence, for
$$s \geq u = \max \sett{\frac{2\alpha^2\log (2n)}{\xi^2{(1-\alpha)}^2{(\riskgap{i} - \eeps)}^2}, \frac{\log (2n)}{h \paren{\frac{\sigma_i^2 {(c_\alpha^*)}^2}{\paren{\sigma_i c_\alpha^* - (1 - \xi)\paren{\riskgap{i} - \eeps}}^2}}}},$$
we have
\begin{align}
\PP_t\paren{G_{is} > \frac{1}{n}}
&\leq \PP_t\paren{\hat{\mu}_{i,s} - \sqrt{\frac{2 \log (2n)}{s}} \leq  \mu_i - \frac{\xi(\riskgap{i} - \eeps)(1-\alpha)}{\alpha}} \nonumber\\*
&\phantom{--}+ \PP_t\paren{\inv{h_-} \paren{\frac{\log (2n)}{s}} \leq \frac{\hat{\sigma}_{i,s}^2 {(c_\alpha^*)}^2}{\paren{\sigma_i c_\alpha^* - (1 - \xi)\paren{\riskgap{i} - \eeps}}^2} \leq \inv{h_+} \paren{\frac{\log (2n)}{s}}} \nonumber\\
&\leq \PP_t\paren{\hat{\mu}_{i,s}- \mu_i \leq  \sqrt{\frac{2 \log (2n)}{s}} - \frac{\xi(\riskgap{i} - \eeps)(1-\alpha)}{\alpha}} \nonumber\\
&\phantom{--}+ \PP_t\paren{ \hat{\sigma}_{i,s}^2 \leq \frac{\paren{\sigma_i c_\alpha^* - (1 - \xi)\paren{\riskgap{i} - \eeps}}^2}{{(c_\alpha^*)}^2} \inv{h_+} \paren{\frac{\log (2n)}{s}}} \nonumber\\
&\leq \exp \paren{-\frac{s \paren{\frac{\xi(\riskgap{i} - \eeps)(1-\alpha)}{\alpha} -  \sqrt{\frac{2 \log (2n)}{s}}}^2}{2 \sigma_i^2}} \nonumber\\*
&\phantom{--}+ \exp\paren{-(s-1) \frac{\paren{\sigma_i^2 - \frac{\paren{\sigma_i c_\alpha^* - (1 - \xi)\paren{\riskgap{i} - \eeps}}^2}{{(c_\alpha^*)}^2} \inv{h_+} \paren{\frac{\log (2n)}{s}}}^2}{4 \sigma_i^4}} \label{step: mosart}\\
&\leq \exp \paren{-\frac{s\eeps^2}{\sigma_i^2}} + \exp\paren{-(s-1) \frac{\eeps^4}{\sigma_i^2}}, \nonumber
\end{align} where (\ref{step: mosart}) follows from Lemma~\ref{lem: laurent}, which we state below. Summing over $s$,
\begin{align*}
\sum_{s=1}^n \PP_t\paren{G_{is} > \frac{1}{n}}
&\leq u + \sum_{s=\lceil u \rceil}^n \parenb{\exp\paren{-\frac{s\eeps^2}{\sigma_i^2}} + \exp\paren{-(s-1) \frac{\eeps^4}{\sigma_1^2}}}\\
&\leq 1 + \max \sett{\frac{2\alpha^2\log (2n)}{\xi^2{(1-\alpha)}^2\paren{\riskgap{i} - \eeps}^2}, \frac{\log (2n)}{h \paren{\frac{\sigma_i^2 {(c_\alpha^*)}^2}{\paren{\sigma_i  c^*- (1 - \xi)\paren{\riskgap{i} - \eeps}}^2}}}} + \frac{C_5}{\eeps^4} + \frac{C_6}{\eeps^2}.
\end{align*}
\begin{lemma}[\citet{laurent2000}]
\label{lem: laurent}
For any $X \Fol \chi_{s-1}^2$, $\displaystyle \PP(X \leq x) \leq \exp \paren{-\frac{{(s-1-x)}^2}{4(s-1)}}.$
\end{lemma}
\input{raw/proof_of_F_to_0_2}
\end{proof}

%% file: raw/proof_of_F_to_0_2.tex
Setting $$\xi_\alpha = 1 - \frac{\sigma_i c_\alpha^*}{\riskgap{i}}\paren{1 - \frac{1}{\sqrt{\inv{h_+}(1/A_{\alpha,1}^i)}}},$$ we observe that $\omega_{i,\alpha} := \sigma_i\paren{1 - \frac{1}{\sqrt{\inv{h_+}(1/A_{\alpha,1}^i)}}} \to \sigma_i \paren{1 - \frac{1}{\sqrt{\inv{h_+}({(\mu_i-\mu_1)}^2/2)}}} =: \Omega_i$ since $A_{\alpha,1} \to \frac{2}{{(\mu_i-\mu_1)}^2}$, and hence, $$\xi_\alpha = 1-\frac{\omega_{i,\alpha} c_\alpha^*}{\riskgap{i}} = 1 - \frac{\omega_{i,\alpha} c_\alpha^*(1-\alpha)}{(\mu_i-\mu_1)\alpha + \sigma_i c_\alpha^*(1-\alpha)} \to 1 - \frac{\Omega_i \cdot 0}{\mu_i - \mu_1} = 1$$ as $\alpha \to 1^-$. Furthermore, we have
\begin{align*}
B_{\alpha,\xi_\alpha}^i
&= \frac{1}{h \paren{\frac{\sigma_i^2 {(c_\alpha^*)}^2}{{\paren{\sigma_i c_\alpha^* -(1-\xi_\alpha)\riskgap{i}}}^2}}} = \frac{1}{h \paren{\frac{\sigma_i^2 {(c_\alpha^*)}^2}{{\paren{\sigma_i c_\alpha^* -\omega c_\alpha^*}}^2}}} = \frac{1}{h \paren{\frac{\sigma_i^2 }{{\paren{\sigma_i  -\omega }}^2}}}\\
&= \frac{1}{h \paren{\frac{\sigma_i^2 }{\sigma_i^2 / \inv{h_+}(1/A_{\alpha,1}^i)}}} = \frac{1}{h \paren{\inv{h_+}(1/A_{\alpha,1}^i)}} = \frac{1}{1/A_{\alpha,1}^i} = A_{\alpha,1}^i \leq A_{\alpha,\xi_\alpha}^i.
\end{align*}

%% file: raw/raw_supp/inf_proof_3.tex
\section{Proof of Theorem~\ref{thm: inf_regret}}
\label{sec: S-2}
The proof of Theorem~\ref{thm: inf_regret} is similar to that of Theorem~\ref{thm: risk_regret}.
\begin{proof}[Proof of Theorem~\ref{thm: inf_regret}]
Fix $\eeps > 0$, and define
\begin{equation*}
	E_i(t) := \sett{\hat{c}_{\alpha}(i,t) > \tau + \eeps},
\end{equation*}
By Lemma~\ref{subthm: lattimore_main}, and the linearity of expectation, we can split $\EE[T_{i,n}]$ into two parts as  follows
\begin{equation}
\label{eqn: key_lemma_inf}
\EE[T_{i,n}] \leq \sum_{s=0}^{n-1} \EE \parenb{\frac{1}{G_{1s}} - 1} +  \sum_{s=0}^{n-1} \PP \paren{G_{is} > \frac{1}{n}} + 1.\end{equation}
By Lemmas~\ref{lem: upper_bound_term_1_inf} and \ref{lem: upper_bound_term_2_inf}, we have
\begin{align*}
\sum_{s=1}^{n} \EE \parenb{\frac{1}{G_{1s}} - 1} &\leq \frac{C_7}{\eeps^2} + C_8,\ \text{and}\\
\sum_{s=1}^n \PP_t \paren{G_{is} > \frac{1}{n}} &\leq 1 + \max \sett{\frac{2 \alpha^2 \log (2n)}{\xi^2{(1-\alpha)}^2\paren{{\infgap{\tau}{i,\alpha} - \eeps}}^2}, \frac{\log (2n)}{h \paren{\frac{\sigma_i^2 {(c_\alpha^*)}^2}{{\paren{\sigma_i c_\alpha^* -(1-\xi)(\infgap{\tau}{i,\alpha} - \eeps)}}^2}}}} + \frac{C_9}{\eeps^4} + \frac{C_{10}}{\eeps^2}.
\end{align*}
Plugging the two displays into (\ref{eqn: key_lemma_inf}), we have
\begin{equation}
	\EE[T_{i,n}] \leq 1 + \max \sett{\frac{2 \alpha^2 \log (2n)}{\xi^2{(1-\alpha)}^2\paren{{\infgap{\tau}{i,\alpha} - \eeps}}^2}, \frac{\log (2n)}{h \paren{\frac{\sigma_i^2 {(c_\alpha^*)}^2}{{\paren{\sigma_i c_\alpha^* -(1-\xi)(\infgap{\tau}{i,\alpha} - \eeps)}}^2}}}} + \frac{C_6'}{\eeps^4} + \frac{C_7'}{\eeps^2} + C_8', \label{eqn: pre_result_inf}
\end{equation}
where $C_6',C_7',C_8'$ are constants. Setting $\eeps = {(\log n)}^{-\frac{1}{8}}$ into (\ref{eqn: pre_result_inf}), we get
\begin{align*}
\limsup_{n\to\infty}\frac{\infreg{n}{\text{CVaR-TS}}}{\log n}
&\leq \sum_{i \in \kay_{\tau}^c} \paren{\max \sett{\frac{2\alpha^2}{\xi^2{(1-\alpha)}^2\Delta_{\tau}^2(i)}, \frac{1}{h \paren{\frac{\sigma_i^2 {(c_\alpha^*)}^2}{{\paren{\sigma_i c_\alpha^* -(1-\xi)(\infgap{\tau}{i,\alpha})}}^2}}}}}\infgap{\tau}{i,\alpha}\\
&=\sum_{i \in \kay_{\tau}^c} \max\sett{D_{\alpha,\xi}^i,E_{\alpha,\xi}^i} \infgap{\tau}{i,\alpha}=\sum_{i \in \sdiff{[K]}{\kay^*}} F_{\alpha,\xi}^i\infgap{\tau}{i,\alpha}.
\end{align*}
\end{proof}
\begin{lemma}[Upper bounding the first term of (\ref{eqn: key_lemma_inf})]
\label{lem: upper_bound_term_1_inf}
We have $$\sum_{s=1}^{n} \EE \parenb{\frac{1}{G_{1s}}-1} \leq \frac{C_7}{\eeps^2} + C_8,$$ where $C_7,C_8$ are constants.
\end{lemma}
\begin{proof}[Proof of Lemma~\ref{lem: upper_bound_term_1_inf}]
We note that $i$ is either a deceiver arm or a non-deceiver arm.
\begin{enumerate}[label=\qalph]
	\item Suppose $i$ is a deceiver arm. Then
$$E_{i}(t) = \sett{\hat{c}_{\alpha}(i,t) > \tau + \eeps} \cap \sett{\hat{\mu}_{i,t} \leq \mu_1 - \eeps} \To E_{i}^c(t) \supseteq \sett{\hat{\mu}_{i,t} > \mu_1 -\eeps}.$$
This allows us to establish lower bound
\begin{equation}
\PP_t(E_1^c(t)|T_{1,t} = s, \hat{\mu}_{1,s} = \mu, \hat{\sigma}_{1,s} = \sigma)
\geq \begin{cases}
\frac{1}{2} & \text{if $\mu \geq \mu_1$}\\
  \PP_t(\theta_{1,t} > \mu_1 - \eeps) & \text{if $\mu < \mu_1$.}
 \end{cases} \label{step: deceiver_lower_bound}
 \end{equation}
 Define the conditional version of $G_{1s}$ by $$\widetilde{G}_{1s} = G_{1s} |_{\hat{\mu}_{1,s},\hat{\sigma}_{1,s} = \beta} = \PP_t(\hat{\mu}_{1,t} > \mu_1 - \eeps\ \text{or}\ \hat{c}_{\alpha}(1,t) \leq \tau + \eeps \mid \hat{\mu}_{1,s},\hat{\sigma}_{1,s} = \beta).$$
Define $c_1 = \frac{1}{\sqrt{2\pi \sigma_1^2}}$ and $c_2 = \frac{1}{2^{\frac{s}{2}}\Gamma(s/2) \sigma_1^s}$. We partition the parameter space $$(\beta,\mu) \in [0,\infty) \times (-\infty,\infty) = A' \cup B'$$ where $A' = [0,\infty) \times [\mu_1,\infty)$ and $B' = [0,\infty) \times (-\infty,\mu_1]$. We can then partition $\EE \parenb{\frac{1}{G_{1s}}-1}$ into two parts:
\begin{align*}
\EE \parenb{\frac{1}{G_{1s}}-1}
&= c_1 c_2 \rint{0}{\infty}{\rint{-\infty}{\infty}{\frac{1-\widetilde{G}_{1s}}{\widetilde{G}_{1s}}\exp \paren{-\frac{{(\mu-\mu_1)}^2}{2\sigma_1^2}}\beta^{s-1}e^{-\frac{\beta^2}{2\sigma^2}}}{\mu}}{\beta}\\
&= c_1 c_2 \paren{\int_{A'} + \int_{B'}}\frac{1-\widetilde{G}_{1s}}{\widetilde{G}_{1s}}\exp \paren{-\frac{{(\mu-\mu_1)}^2}{2\sigma_1^2}}\beta^{s-1}e^{-\frac{\beta^2}{2\sigma^2}}\ \mathrm{d}\mu\  \mathrm{d}\beta.
\end{align*}
For Part $A'$, using the first case in (\ref{step: deceiver_lower_bound}),
$$\frac{1-\widetilde{G}_{1s}}{\widetilde{G}_{1s}} \leq 2(1 - \widetilde{G}_{1s}).$$
Thus,
\begin{align}
&c_1 c_2 \int_{A'} \frac{1-\widetilde{G}_{1s}}{\widetilde{G}_{1s}} \exp \paren{-\frac{s{(\mu-\mu_1)}^2}{2 \sigma_1^2}}\beta^{s-1} e^{-\frac{\beta^2}{2 \sigma^2}} \mathrm{d}\mu \mathrm{d}\beta \nonumber\\
&\leq 2c_1 c_2 \int_{A'} (1 - \widetilde{G}_{1s})\exp \paren{-\frac{s{(\mu-\mu_1)}^2}{2 \sigma_1^2}}\beta^{s-1} e^{-\frac{\beta^2}{2 \sigma^2}} \mathrm{d}\mu \mathrm{d}\beta \nonumber\\
&\leq 2c_1 c_2 \int_{A'} \PP_t\paren{\theta_{1,t} \leq \mu_1 - \eeps} \PP_t\paren{\hat{c}_{\alpha}(1,t) > \tau + \eeps}\exp \paren{-\frac{s{(\mu-\mu_1)}^2}{2 \sigma_1^2}}\beta^{s-1} e^{-\frac{\beta^2}{2 \sigma^2}} \mathrm{d}\mu \mathrm{d}\beta \nonumber\\
&\leq 2c_1 \rint{\mu_1}{\infty}{\exp \paren{-\frac{s}{2}{(\mu - \mu_1 + \eeps)}^2} \exp \paren{-\frac{s{(\mu-\mu_1)}^2}{2 \sigma_1^2}}}{\mu} \cdot c_2 \rint{0}{\infty}{\beta^{s-1} e^{-\frac{\beta^2}{2 \sigma^2}}}{\beta} \label{step: deceiver_intermediate}\\
&\leq 2 \exp \paren{-\frac{s\eeps^2}{4}}, \nonumber
\end{align}
where (\ref{step: deceiver_intermediate}) follows from using tail bounds on the Gaussian distribution.

For Part $B'$, using the second case (\ref{step: deceiver_lower_bound}),
$$\frac{1-\widetilde{G}_{1s}}{\widetilde{G}_{1s}} \leq \frac{1}{\PP_t(\theta_{1,t} > \mu_1 - \eeps | \hat{\mu}_{1,s} = \mu)}.$$
Then, reusing a calculation in \citet{zhu2020thompson},
\begin{align*}
&c_1 c_2 \int_{B'} \frac{1-\widetilde{G}_{1s}}{\widetilde{G}_{1s}} \exp \paren{-\frac{s{(\mu-\mu_1)}^2}{2 \sigma_1^2}}\beta^{s-1} e^{-\frac{\beta^2}{2 \sigma^2}} \, \mathrm{d}\mu  \, \mathrm{d}\beta\\
&\leq c_1 c_2 \int_{B'} \frac{1}{\PP_t(\theta_{1,t} > \mu_1 - \eeps | \hat{\mu}_{1,s} = \mu)} \exp \paren{-\frac{s{(\mu-\mu_1)}^2}{2 \sigma_1^2}}\beta^{s-1} e^{-\frac{\beta^2}{2 \sigma^2}} \, \mathrm{d}\mu \, \mathrm{d}\beta\\
&\leq c_1 \int_{-\infty}^{\mu_1-\eeps}{\frac{1}{\PP_t(\theta_{1,t} > \mu_1 - \eeps | \hat{\mu}_{1,s} = \mu)} \exp \paren{-\frac{s{(\mu-\mu_1)}^2}{2 \sigma_1^2}}} \, \mathrm{d}{\mu} \cdot c_2 \int_{0}^{\infty}{\beta^{s-1} e^{-\frac{\beta^2}{2 \sigma^2}}} \, \mathrm{d}{\beta}\\
&\leq \frac{\sqrt{2}}{\sqrt{\pi}} \cdot \frac{3}{\sqrt{2\pi} \eeps^2 s\sqrt{s}} \cdot \exp \paren{-\frac{s\eeps^2}{2}} = \frac{3}{\pi \eeps^2 s\sqrt{s}} \exp \paren{-\frac{s\eeps^2}{2}}.
\end{align*}
Combining both parts, we can upper bound $\EE \parenb{\frac{1}{G_{1s}}-1}$ by $$\EE \parenb{\frac{1}{G_{1s}}-1} \leq 2 \exp \paren{-\frac{s\eeps^2}{4}}+ \frac{3}{\pi \eeps^2 s\sqrt{s}} \exp \paren{-\frac{s\eeps^2}{2}}.$$
Summing over $s$, we have $$\sum_{s=1}^{n} \EE \parenb{\frac{1}{G_{1s}}-1} \leq \frac{C_7^{(a)}}{\eeps^2} + C_8^{(a)}.$$
\item\label{sub_reg_ref_1} Suppose $i$ is a non-deceiver arm. Then
$$E_{i}(t) = \sett{\hat{c}_{\alpha}(i,t) > \tau + \eeps} \cap \sett{\hat{\mu}_{i,t} \geq \mu_1 + \eeps} \To E_{i}^c(t) \supseteq \sett{\hat{\mu}_{i,t} < \mu_1 +\eeps}.$$
This allows us to establish lower bound
\begin{equation}
	\PP_t(E_1^c(t)|T_{1,t} = s, \hat{\mu}_{1,s} = \mu, \hat{\sigma}_{1,s} = \sigma)
\geq \begin{cases}
\frac{1}{2} & \text{if $\mu < \mu_1$}\\
  \PP_t(\theta_{1,t} < \mu_1 + \eeps) & \text{if $\mu \geq \mu_1$.}
 \end{cases} \label{step: infeasible_lower_bound}
\end{equation}
Define the conditional version of $G_{1s}$ by $$\widetilde{G}_{1s} = G_{1s} |_{\hat{\mu}_{1,s},\hat{\sigma}_{1,s} = \beta} = \PP_t(\hat{\mu}_{1,t} \leq \mu_1 + \eeps\ \text{or}\ \hat{c}_{\alpha}(1,t) > \tau - \eeps \mid \hat{\mu}_{1,s},\hat{\sigma}_{1,s} = \beta).$$
Define $c_1 = \frac{1}{\sqrt{2\pi \sigma_1^2}}$ and $c_2 = \frac{1}{2^{\frac{s}{2}}\Gamma(s/2) \sigma_1^s}$. We partition the parameter space $$(\beta,\mu) \in [0,\infty) \times (-\infty,\infty) = A'' \cup B''$$ where $A'' = [0,\infty) \times [\mu_1,\infty)$ and $B'' = [0,\infty) \times (-\infty,\mu_1]$. We can then partition $\EE \parenb{\frac{1}{G_{1s}}-1}$ into two parts:
\begin{align*}
\EE \parenb{\frac{1}{G_{1s}}-1}
&= c_1 c_2 \rint{0}{\infty}{\rint{-\infty}{\infty}{\frac{1-\widetilde{G}_{1s}}{\widetilde{G}_{1s}}\exp \paren{-\frac{{(\mu-\mu_1)}^2}{2\sigma_1^2}}\beta^{s-1}e^{-\frac{\beta^2}{2\sigma^2}}}{\mu}}{\beta}\\
&= c_1 c_2 \paren{\int_{A''} + \int_{B''}}\frac{1-\widetilde{G}_{1s}}{\widetilde{G}_{1s}}\exp \paren{-\frac{{(\mu-\mu_1)}^2}{2\sigma_1^2}}\beta^{s-1}e^{-\frac{\beta^2}{2\sigma^2}}\ \mathrm{d}\mu\  \mathrm{d}\beta.
\end{align*}
For Part $B''$, using the first case in (\ref{step: infeasible_lower_bound}),
$$\frac{1-\widetilde{G}_{1s}}{\widetilde{G}_{1s}} \leq 2(1 - \widetilde{G}_{1s}).$$
Thus,
\begin{align*}
&c_1 c_2 \int_{B''} \frac{1-\widetilde{G}_{1s}}{\widetilde{G}_{1s}} \exp \paren{-\frac{s{(\mu-\mu_1)}^2}{2 \sigma_1^2}}\beta^{s-1} e^{-\frac{\beta^2}{2 \sigma^2}} \,  \mathrm{d}\mu  \, \mathrm{d}\beta\\
&\leq 2c_1 c_2 \int_{B''} (1 - \widetilde{G}_{1s})\exp \paren{-\frac{s{(\mu-\mu_1)}^2}{2 \sigma_1^2}}\beta^{s-1} e^{-\frac{\beta^2}{2 \sigma^2}} \,  \mathrm{d}\mu  \, \mathrm{d}\beta\\
&\leq 2c_1 c_2 \int_{B''} \PP_t\paren{\theta_{1,t} \geq \mu_1 + \eeps} \PP_t\paren{\hat{c}_{\alpha}(1,t) \leq \tau - \eeps}\exp \paren{-\frac{s{(\mu-\mu_1)}^2}{2 \sigma_1^2}}\beta^{s-1} e^{-\frac{\beta^2}{2 \sigma^2}} \,  \mathrm{d}\mu  \, \mathrm{d}\beta\\
&\leq 2c_1 \int_{-\infty}^{\mu_1} \PP_t\paren{\theta_{1,t} \geq \mu_1 + \eeps}\mathrm{d}\mu \cdot \int_{0}^{\infty} c_2 \exp \paren{-\frac{s{(\mu-\mu_1)}^2}{2 \sigma_1^2}}\beta^{s-1} e^{-\frac{\beta^2}{2 \sigma^2}}  \,  \mathrm{d}\beta\\
&\leq 2 \exp \paren{-\frac{s\eeps^2}{4}}.
\end{align*}

For Part $A''$, using the second case in (\ref{step: infeasible_lower_bound}), $$\frac{1 - \widetilde{G}_{1s}}{\widetilde{G}_{1s}} \leq  \frac{1}{\PP_t (\theta_{1,t} < \mu_1 + \eeps | \hat{\mu}_{1,s} = \mu)}.$$
By reusing a calculation from Part $A$ of the proof of Lemma~\ref{lem: upper_bound_term_1_risk},
\begin{align*}
&c_1 c_2 \int_{A''} \frac{1-\widetilde{G}_{1s}}{\widetilde{G}_{1s}} \exp \paren{-\frac{s{(\mu-\mu_1)}^2}{2 \sigma_1^2}}\beta^{s-1} e^{-\frac{\beta^2}{2 \sigma^2}} \,  \mathrm{d}\mu  \, \mathrm{d}\beta\\
&\leq c_1 c_2 \int_{A''} \frac{1}{\PP_t(\theta_{1,t} - \mu_1 \leq \eeps | \hat{\mu}_{1,s} = \mu)} \exp \paren{-\frac{s{(\mu-\mu_1)}^2}{2 \sigma_1^2}}\beta^{s-1} e^{-\frac{\beta^2}{2 \sigma^2}}  \, \mathrm{d}\mu  \, \mathrm{d}\beta \\
&\leq c_1 \int_{\mu_1+\eeps}^{\infty} \frac{1}{\PP_t(\theta_{1,t} - \mu_1 \leq \eeps | \hat{\mu}_{1,s} = \mu)} \exp \paren{-\frac{s{(\mu-\mu_1)}^2}{2 \sigma_1^2}}  \, \mathrm{d}\mu \cdot c_2 \int_0^{\infty} \beta^{s-1} e^{-\frac{\beta^2}{2 \sigma^2}}  \, \mathrm{d}\beta\\
&\leq \frac{3}{\pi \eeps^2 s \sqrt{s}} \exp \paren{-\frac{s\eeps^2}{2}}.
\end{align*}
Combining both parts, we can upper bound $\EE \parenb{\frac{1}{G_{1s}}-1}$ by $$\EE \parenb{\frac{1}{G_{1s}}-1} \leq 2 \exp \paren{-\frac{s\eeps^2}{4}}+ \frac{3}{\pi \eeps^2 s\sqrt{s}} \exp \paren{-\frac{s\eeps^2}{2}}.$$
Summing over $s$, we have $$\sum_{s=1}^{n} \EE \parenb{\frac{1}{G_{1s}}-1} \leq \frac{C_7^{(b)}}{\eeps^2} + C_8^{(b)}.$$
\end{enumerate}
Setting $C_7 = \max \sett{C_7^{(a)},C_7^{(b)}}$ and $C_8 = \max \sett{C_8^{(a)},C_8^{(b)}}$, we get $$\sum_{s=1}^{n} \EE \parenb{\frac{1}{G_{1s}}-1} \leq \frac{C_7}{\eeps^2} + C_8.$$
\end{proof}
\begin{lemma}[Upper bounding the second term of (\ref{eqn: key_lemma_inf})]
\label{lem: upper_bound_term_2_inf}
For $\xi \in (0,1)$,
\begin{align*}
&\PP(E_i^c(t) | T_{i,t} = s, \hat{\mu}_{i,t} = \mu, \hat{\sigma}_{i,t} = \sigma) \nonumber\\
&\leq \exp\paren{-\frac{s}{2} \paren{\mu-\mu_i +\frac{ \xi (1-\alpha)(\infgap{\tau}{i,\alpha}-\eeps)}{\alpha}}^2} + \exp \paren{-sh \paren{\frac{\sigma^2 {(c_\alpha^*)}^2}{{\paren{\sigma_i c_\alpha^* -(1-\xi)(\infgap{\tau}{i,\alpha} - \eeps)}}^2}}}. \label{step: tail_upper_bound_inf}
\end{align*}
Furthermore,
\begin{equation*}
\sum_{s=1}^n \PP_t \paren{G_{is} > \frac{1}{n}} \leq 1 + \max \sett{\frac{2 \alpha^2 \log (2n)}{\xi^2{(1-\alpha)}^2\paren{{\infgap{\tau}{i,\alpha} - \eeps}}^2}, \frac{\log (2n)}{h \paren{\frac{\sigma_i^2 {(c_\alpha^*)}^2}{{\paren{\sigma_i c_\alpha^* -(1-\xi)(\infgap{\tau}{i,\alpha} - \eeps)}}^2}}}} + \frac{C_9}{\eeps^4} + \frac{C_{10}}{\eeps^2}, \label{step: upper_bound_inf_regret}
\end{equation*}
where $C_9,C_{10}$ are constants.
\end{lemma}
\begin{proof}[Proof of Lemma~\ref{lem: upper_bound_term_2_inf}] The result follows immediately from Lemma~\ref{lem: upper_bound_term_2_risk} by replacing $(c_\alpha(1),\riskgap{i})$ with $(\tau,\infgap{\tau}{i,\alpha})$. Choose 
$$\xi_\alpha = 1 - \frac{\sigma_i c_\alpha^*}{\infgap{\tau}{i,\alpha}}\paren{1 - \frac{1}{\sqrt{\inv{h_+}(1/D_{\alpha,1}^i)}}}$$ and
	repeat the argument in Theorem~\ref{thm: risk_regret} by replacing $(c_\alpha(1), \riskgap{i})$ with $(\tau, \infgap{\tau}{i,\alpha})$ to get the desired result.
\end{proof}

%% file: raw/raw_supp/sub_proof.tex
\section{Proof of Theorem~\ref{thm: sub_regret}}
\label{sec: S-3}
The proof of Theorem~\ref{thm: sub_regret} follows the same strategy and is even more straightforward. Nevertheless, we include it here for completeness.

\begin{proof}[Proof of Theorem~\ref{thm: sub_regret}:] For any arm $i \in \sdiff{\kay_{\tau}}{\kay^*}$, define the good event by 
$$E_i(t) = \sett{\hat{\mu}_{i,t} > \mu_1+\eeps}.$$
We first observe that $$E_i^c(t) = \sett{\hat{\mu}_{i,t} \leq \mu_1+\eeps}.$$
By Lemma~\ref{subthm: lattimore_main}, and the linearity of expectation, we can divide $\EE[T_{i,n}]$ into two parts as 
\begin{equation}
\label{eqn: key_lemma_sub}
\EE[T_{i,n}] \leq \sum_{s=0}^{n-1} \EE \parenb{\frac{1}{G_{1s}} - 1} +  \sum_{s=0}^{n-1} \PP \paren{G_{is} > \frac{1}{n}} + 1.\end{equation}
By the computations in part \ref{sub_reg_ref_1} of the proof of Lemma~\ref{lem: upper_bound_term_1_inf}, we have
\begin{equation}\label{step: sub_term_1}
\sum_{s=1}^{n} \EE \parenb{\frac{1}{G_{1s}}-1} \leq \frac{C_{11}}{\eeps^2} + C_{12}.
\end{equation}
By a Gaussian concentration bound, $$\PP_t(E_i^c(t)) = \PP_t\paren{\theta_{i,t} - \mu \leq -(\mu - \mu_1 - \eeps)} \leq \exp \paren{-\frac{s}{2}\paren{\mu - \mu_1 - \eeps}^2}.$$
By similar computations as in \citet{zhu2020thompson},
\begin{equation}\label{step: sub_term_2}
\sum_{s=1}^n \PP_t \paren{G_{is} > \frac{1}{n}} \leq 1 + \frac{2 \log(n)}{{(\mu_i-\mu_1-\eeps)}^2} + \frac{C_{13}}{\eeps^4} + \frac{C_{14}}{\eeps^2}.
\end{equation}
Plugging (\ref{step: sub_term_1}) and (\ref{step: sub_term_2}) into (\ref{eqn: key_lemma_sub}), we have the expected number of pulls on arm $i$ up to round $n$ given by $$\EE[T_{i,n}] \leq \frac{2 \log(n)}{{(\mu_i-\mu_1-\eeps)}^2} + \frac{C_9'}{\eeps^4} + \frac{C_{10}'}{\eeps^2} + C_{11}'.$$
Setting $\eeps = {(\log n)}^{-\frac{1}{8}}$, we get the following result for the suboptimality regret:
$$\limsup_{n\to\infty}\frac{\subreg{n}{\text{CVaR-TS}}}{\log n} \leq \sum_{i \in \sdiff{\kay_{\tau}}{\kay^*}} \frac{2}{{(\mu_i - \mu_1)}^2}\subgap{i}.$$
\end{proof}

%% file: raw/raw_supp/lower_bound_risk.tex
\section{Proof of Theorem~\ref{thm: lower_bounds}}
\label{sec: S-4}
\begin{proof}[Proof of Theorem~\ref{thm: lower_bounds}:]
Let $\pi$ be any (risk-, infeasibility-, or suboptimality-)consistent algorithm. We will prove the theorem in two parts, when the instances are infeasible and feasible respectively.
\begin{enumerate}
	\item Fix the threshold level $\tau \in \RR$. For any arm $i$ with distribution $\nu(i) \Fol {\cal N}(\mu_i,\sigma_i^2)$ in a given    infeasible instance $(\nu, \tau)$, define ${\cal S}_i = \sett{\nu'(i) \in {\cal E}_{\cal N}^K : \CVaR{\alpha}{\nu'(i)} < \CVaR{\alpha}{1}}$ and
	\begin{align*}
	\eta(i,\alpha) &= \inf_{\nu'(i) \in {\cal S}_i}\sett{\mathrm{KL}(\nu(i),\nu'(i))} = \inf_{\nu'(i) \in {\cal S}_i} \sett{\log \frac{\sigma_i'}{\sigma_i} + \frac{\sigma_i^2 + {(\mu_i - \mu_i')}^2}{2 {(\sigma_i')}^2} - \frac{1}{2}},
	\end{align*}
	where $\nu'(i) \Fol {\cal N}(\mu_i',{(\sigma_i')}^2)$ and the KL-divergence of two Gaussians is well-known. By Theorem 4 in \citet{kagrecha2020constrained},
	the expected number of pulls of a non-optimal arm $i$ is characterised by $$\liminf_{n \to \infty} \frac{\EE[T_{i,n}]}{\log n} \geq \frac{1}{\eta(i,\alpha)}.$$
	Fix $\eeps > 0$ and consider the arm $\nu'(i)$ with distribution ${\cal N}\paren{\mu_i - \frac{\sqrt{2}}{\xi \sqrt{A_{\alpha,\xi}^i}}-\eeps,\sigma_i^2}$. Then a direct computation gives
	\begin{align*}
		\CVaR{\alpha}{\nu'(i)} -\CVaR{\alpha}{1} &= \paren{\mu_i -\frac{\sqrt{2}}{\xi \sqrt{A_{\alpha,\xi}^i}}-\varepsilon}\paren{\frac{\alpha}{1-\alpha}} + \sigma_i c_\alpha^*-\mu_1\paren{\frac{\alpha}{1-\alpha}} - \sigma_1 c_\alpha^*\\
		&= \mu_i \paren{\frac{\alpha}{1-\alpha}} - \riskgap{i}   - \mu_1 \paren{\frac{\alpha}{1-\alpha}} +(\sigma_i - \sigma_1)c_\alpha^*-\varepsilon\paren{\frac{\alpha}{1-\alpha}}\\
		&= - \riskgap{i}  + \riskgap{i}-\varepsilon\paren{\frac{\alpha}{1-\alpha}}\\
		&= -\varepsilon\paren{\frac{\alpha}{1-\alpha}}\\
		& < 0.
	\end{align*}
	Thus, $\CVaR{\alpha}{\nu'(i)} < \CVaR{\alpha}{1}$ and $\nu'(i) \in {\cal S}_i$. Furthermore,
	\begin{align*}
	\mathrm{KL}(\nu(i),\nu'(i))
	&= {\log \frac{\sigma_i}{\sigma_i} + \frac{\sigma_i^2 + \paren{\mu_i - \paren{\mu_i - \frac{\sqrt{2}}{\xi \sqrt{A_{\alpha,\xi}^i}}-\eeps}}^2}{2 \sigma_i^2} - \frac{1}{2}} = \frac{1 }{\xi^2 \sigma_i^2 A_{\alpha,\xi}^i} - \frac{\eeps}{2\sigma_i^2} \paren{\frac{2\sqrt{2}}{\xi \sqrt{A_{\alpha,\xi}^i}} - \eeps}.
	\end{align*}
	By the definition of $\eta$, $$\eta(i,\alpha) \leq \lim_{\eeps \to 0^+}\parenb{\frac{1 }{\xi^2 \sigma_i^2 A_{\alpha,\xi}^i} - \frac{\eeps}{2\sigma_i^2} \paren{\frac{2\sqrt{2}}{\xi \sqrt{A_{\alpha,\xi}^i}} - \eeps}}=\frac{1 }{\xi^2 \sigma_i^2 A_{\alpha,\xi}^i} \quad\Longrightarrow\quad  \frac{1}{\eta(i,\alpha)} \geq \xi^2\sigma_i^2 A_{\alpha,\xi}^i.$$
	Hence, $$\liminf_{n \to \infty} \frac{\riskreg{n}{\pi}}{\log n} =\sum_{i \in \sdiff{[K]}{\kay^*}}\paren{\liminf_{n \to \infty} \frac{\EE[T_{i,n}] }{\log n}}\riskgap{i} \geq \sum_{i \in \sdiff{[K]}{\kay^*}} \xi^2 \sigma_i^2 A_{\alpha,\xi}^i \riskgap{i}.$$
	\item Fix the threshold level $\tau \in \RR$. For any arm $i$ with distribution $\nu(i) \Fol {\cal N}(\mu_i,\sigma_i^2)$ in a given feasible instance $(\nu, \tau)$, define ${\cal V}_i = \sett{\nu'(i) \in {\cal E}_{\cal N}^K : \mu(\nu'(i)) < \mu_1\ \text{and}\  \CVaR{\alpha}{\nu'(i)} \leq \tau}$ and
	\begin{align*}
	\eta(i,\alpha) &= \inf_{\nu'(i) \in {\cal V}_i}\sett{\mathrm{KL}(\nu(i),\nu'(i))} = \inf_{\nu'(i) \in {\cal V}_i} \sett{\log \frac{\sigma_i'}{\sigma_i} + \frac{\sigma_i^2 + {(\mu_i - \mu_i')}^2}{2 {(\sigma_i')}^2} - \frac{1}{2}},
	\end{align*}
	where $\nu'(i) \Fol {\cal N}(\mu_i',{(\sigma_i')}^2)$ and the KL-divergence of two Gaussians is well-known. By Theorem 4 in \citet{kagrecha2020constrained},
	the expected number of pulls of a non-optimal arm $i$ is characterised by $$\liminf_{n \to \infty} \frac{\EE[T_{i,n}]}{\log n} \geq \frac{1}{\eta(i,\alpha)}.$$
	Suppose arm $i$ is infeasible. Fix $\eeps > 0$, and consider the arm $\nu'(i)$ with distribution ${\cal N}\paren{\mu_i - \frac{\sqrt{2}}{\xi \sqrt{D_{\alpha,\xi}^i}}-\varepsilon,\sigma_i^2}$. Then $\mu(\nu'(i)) < \mu_1$ trivially, and a direct computation gives
	\begin{align*}
	\CVaR{\alpha}{\nu'(i)} - \tau &= \paren{\mu_i - \frac{\sqrt{2}}{\xi \sqrt{D_{\alpha,\xi}^i}}-\varepsilon }\paren{\frac{\alpha}{1-\alpha}} + \sigma_i c_\alpha^*-\tau\\
	&= \mu_i \paren{\frac{\alpha}{1-\alpha}} - \infgap{\tau}{i}  + \sigma_i c_\alpha^* - \tau-\varepsilon\paren{\frac{\alpha}{1-\alpha}}\\
	&= \mu_i \paren{\frac{\alpha}{1-\alpha}} - \CVaR{\alpha}{i} + \sigma_i c_\alpha^*-\varepsilon\paren{\frac{\alpha}{1-\alpha}}\\
	&= -\varepsilon\paren{\frac{\alpha}{1-\alpha}}< 0.
	\end{align*}
	Thus, $\CVaR{\alpha}{\nu'(i)} < \tau$ and $\nu'(i) \in {\cal V}_i$. By similar computations  as in Part 1,
	\begin{align*} 
	\mathrm{KL}(\nu(i),\nu'(i))
	&= {\log \frac{\sigma_i}{\sigma_i} + \frac{\sigma_i^2 + \paren{\mu_i - \paren{\mu_i - \frac{\sqrt{2}}{\xi \sqrt{D_{\alpha,\xi}^i}}-\eeps}}^2}{2 \sigma_i^2} - \frac{1}{2}} = \frac{1 }{\xi^2 \sigma_i^2 D_{\alpha,\xi}^i} - \frac{\eeps}{2\sigma_i^2} \paren{\frac{2\sqrt{2}}{\xi \sqrt{A_{\alpha,\xi}^i}} - \eeps}.
	\end{align*}
	By the definition of $\eta$, $$\eta(i,\alpha) \leq \lim_{\eeps \to 0^+} \parenb{\frac{1}{\xi^2 \sigma_i^2 D_{\alpha,\xi}^i} - \frac{\eeps}{2\sigma_i^2} \paren{\frac{2\sqrt{2}}{\xi \sqrt{A_{\alpha,\xi}^i}} - \eeps} }= \frac{1}{\xi^2\sigma_i^2 D_{\alpha,\xi}^i} \quad \Longrightarrow\quad \frac{1}{\eta(i,\alpha)} \geq \xi^2 \sigma_i^2 D_{\alpha,\xi}^i.$$
	Hence, $$\liminf_{n \to \infty} \frac{\infreg{n}{\pi}}{\log n} =\sum_{i \in \kay_\tau^*}\paren{\liminf_{n \to \infty} \frac{\EE[T_{i,n}] }{\log n}}\infgap{\tau}{i}\geq \sum_{i \in \kay_\tau^c} \xi^2 \sigma_i^2 D_{\alpha,\xi}^i \infgap{\tau}{i}.$$

	Now, suppose arm $i$ is feasible but suboptimal.
Fix $\eeps > 0$ and consider the arm $\nu'(i)$ with distribution ${\cal N}(\mu_i-\sigma_i\Delta(i)-\varepsilon,\sigma_i^2)$.
	Then a direct computation gives
	\begin{align*}
	\CVaR{\alpha}{\nu'(i)} - \tau &= (\mu_i-\sigma_i\Delta(i)-\varepsilon)\paren{\frac{\alpha}{1-\alpha}} + \sigma_i c_\alpha^*-\tau\\
	&= \mu_i \paren{\frac{\alpha}{1-\alpha}} - \tau - \sigma_i\Delta(i) \paren{\frac{\alpha}{1-\alpha}}+ \sigma_i c_\alpha^*-\varepsilon\paren{\frac{\alpha}{1-\alpha}}\\
	&= \CVaR{\alpha}{i} - \tau - \sigma_i\Delta(i) \paren{\frac{\alpha}{1-\alpha}}-\varepsilon\paren{\frac{\alpha}{1-\alpha}} < 0.
	\end{align*}
	Thus, $\CVaR{\alpha}{\nu'(i)} < \tau$ and $\nu'(i) \in {\cal V}_i$. By similar computations as in the previous parts,
	\begin{align*}
	\mathrm{KL}(\nu(i),\nu'(i))
	&= {\log \frac{\sigma_i}{\sigma_i} + \frac{\sigma_i^2 + \paren{\mu_i - \paren{\mu_i - \sigma_i\Delta(i)-\eeps}}^2}{2 \sigma_i^2} - \frac{1}{2}}\\
	&= \frac{\sigma_i^2\Delta^2(i)}{2 \sigma_i^2} - \frac{\eeps}{2\sigma_i^2}\paren{2 \sigma_i \Delta(i) - \eeps}\\
	&=\frac{\Delta^2(i)}{2}-\frac{\eeps}{2\sigma_i^2}\paren{2 \sigma_i \Delta(i) - \eeps}.
	\end{align*}
	By the definition of $\eta$, $$\eta(i,\alpha) \leq 
	\lim_{\eeps \to 0^+}\parenb{\frac{\Delta^2(i)}{2}-\frac{\eeps}{2\sigma_i^2}\paren{2 \sigma_i \Delta(i) - \eeps}}=\frac{\Delta^2(i)}{2} \quad \Longrightarrow\quad \frac{1}{\eta(i,\alpha)} \geq \frac{2}{\Delta^2(i)}.$$
	Hence, $$\liminf_{n \to \infty} \frac{\subreg{n}{\pi}}{\log n}=\sum_{i \in \sdiff{\kay_\tau}{\kay^*}}\paren{\liminf_{n \to \infty} \frac{\EE[T_{i,n}] }{\log n}}\subgap{i} \geq \sum_{i \in \sdiff{\kay_\tau}{\kay^*}} \frac{2}{\Delta(i)}.$$
\end{enumerate}
\end{proof}
	

%% file: CVaR_TS_arXiv.bbl
\begin{thebibliography}{27}
\providecommand{\natexlab}[1]{#1}
\providecommand{\url}[1]{\texttt{#1}}
\expandafter\ifx\csname urlstyle\endcsname\relax
  \providecommand{\doi}[1]{doi: #1}\else
  \providecommand{\doi}{doi: \begingroup \urlstyle{rm}\Url}\fi

\bibitem[Abramowitz \& Stegun(1970)Abramowitz and Stegun]{abramowitz+stegun}
Abramowitz, M. and Stegun, I.~A.
\newblock \emph{Handbook of mathematical functions with formulas, graphs, and
  mathematical tables}, volume~55.
\newblock US Government printing office, 1970.

\bibitem[Agrawal \& Goyal(2012)Agrawal and Goyal]{agrawal2012analysis}
Agrawal, S. and Goyal, N.
\newblock {Analysis of Thompson sampling for the multi-armed bandit problem}.
\newblock In \emph{Conference on Learning Theory}, pp.\  39--1, 2012.

\bibitem[Artzner et~al.(1999)Artzner, Delbaen, Eber, and Heath]{Artzner}
Artzner, P., Delbaen, F., Eber, J.-M., and Heath, D.
\newblock Coherent measures of risk.
\newblock \emph{Mathematical Finance}, 9\penalty0 (3):\penalty0 203--228, 1999.

\bibitem[Baudry et~al.(2020)Baudry, Gautron, Kaufmann, and
  Maillard]{baudry2020thompson}
Baudry, D., Gautron, R., Kaufmann, E., and Maillard, O.-A.
\newblock Thompson sampling for {CVaR} bandits.
\newblock \emph{arXiv preprint arXiv:2012.05754}, 2020.

\bibitem[Bhat \& L.A.(2019)Bhat and L.A.]{bhat2019concentration}
Bhat, S.~P. and L.A., P.
\newblock Concentration of risk measures: A {Wasserstein} distance approach.
\newblock In \emph{Advances in Neural Information Processing Systems}, pp.\
  11762--11771, 2019.

\bibitem[Cassel et~al.(2018)Cassel, Mannor, and Zeevi]{cassel18}
Cassel, A., Mannor, S., and Zeevi, A.
\newblock A general approach to multi-armed bandits under risk criteria.
\newblock In \emph{Proceedings of the 31st Conference On Learning Theory}, pp.\
   1295--1306, 2018.

\bibitem[Fournier \& Guillin(2015)Fournier and Guillin]{fournier2015rate}
Fournier, N. and Guillin, A.
\newblock On the rate of convergence in {Wasserstein} distance of the empirical
  measure.
\newblock \emph{Probability Theory and Related Fields}, 162\penalty0
  (3-4):\penalty0 707--738, 2015.

\bibitem[Galichet et~al.(2013)Galichet, Sebag, and
  Teytaud]{galichet2013exploration}
Galichet, N., Sebag, M., and Teytaud, O.
\newblock Exploration vs exploitation vs safety: Risk-aware multi-armed
  bandits.
\newblock In \emph{Asian Conference on Machine Learning}, pp.\  245--260, 2013.

\bibitem[Harremo{\"e}s(2016)]{Harremo_s_2017}
Harremo{\"e}s, P.
\newblock Bounds on tail probabilities for negative binomial distributions.
\newblock \emph{Kybernetika}, 52\penalty0 (6):\penalty0 943--966, 2016.

\bibitem[Howard \& Matheson.(1972)Howard and Matheson.]{Howard72}
Howard, R.~A. and Matheson., J.~E.
\newblock Risk-sensitive {Markov} decision processes.
\newblock \emph{Management Science}, 18\penalty0 (7):\penalty0 356--369, 1972.

\bibitem[Huo \& Fu(2017)Huo and Fu]{huo2017riskaware}
Huo, X. and Fu, F.
\newblock Risk-aware multi-armed bandit problem with application to portfolio
  selection.
\newblock \emph{Royal Society Open Science}, 4\penalty0 (11):\penalty0 171377,
  2017.

\bibitem[Kagrecha et~al.(2020{\natexlab{a}})Kagrecha, Nair, and
  Jagannathan]{kagrecha2020constrained}
Kagrecha, A., Nair, J., and Jagannathan, K.
\newblock Constrained regret minimization for multi-criterion multi-armed
  bandits.
\newblock \emph{arXiv preprint arXiv:2006.09649}, 2020{\natexlab{a}}.

\bibitem[Kagrecha et~al.(2020{\natexlab{b}})Kagrecha, Nair, and
  Jagannathan]{kagrecha2020statistically}
Kagrecha, A., Nair, J., and Jagannathan, K.
\newblock Statistically robust, risk-averse best arm identification in
  multi-armed bandits.
\newblock \emph{arXiv preprint arXiv:2008.13629}, 2020{\natexlab{b}}.

\bibitem[Khajonchotpanya et~al.(2020)Khajonchotpanya, Yilin, and
  Rujeerapaiboon]{kyn2020cvar}
Khajonchotpanya, N., Yilin, X., and Rujeerapaiboon, N.
\newblock A revised approach for risk-averse multi-armed bandits under cvar
  criterion.
\newblock \emph{ScholarBank NUS}, 2020.

\bibitem[L.A. et~al.(2020)L.A., Jagannathan, and Kolla]{a2019concentration}
L.A., P., Jagannathan, K., and Kolla, R.~K.
\newblock Concentration bounds for {CVaR} estimation: The cases of light-tailed
  and heavy-tailed distributions.
\newblock In \emph{International Conference on Machine Learning}, pp.\
  3657--3666, 2020.

\bibitem[Lattimore \& Szepesv{\'a}ri(2020)Lattimore and
  Szepesv{\'a}ri]{lattimore_szepesvari_2020}
Lattimore, T. and Szepesv{\'a}ri, C.
\newblock \emph{Bandit algorithms}.
\newblock Cambridge University Press, 2020.

\bibitem[Laurent \& Massart(2000)Laurent and Massart]{laurent2000}
Laurent, B. and Massart, P.
\newblock Adaptive estimation of a quadratic functional by model selection.
\newblock \emph{Annals of Statistics}, pp.\  1302--1338, 2000.

\bibitem[Lee et~al.(2020)Lee, Park, and Shin]{lee2020learning}
Lee, J., Park, S., and Shin, J.
\newblock Learning bounds for risk-sensitive learning.
\newblock In \emph{Advances in Neural Information Processing Systems}, 2020.

\bibitem[Rockafellar \& Uryasev(2000)Rockafellar and
  Uryasev]{rockafellar2000optimization}
Rockafellar, R.~T. and Uryasev, S.
\newblock Optimization of conditional value-at-risk.
\newblock \emph{Journal of Risk}, 2\penalty0 (3):\penalty0 21--41, 2000.

\bibitem[Sani et~al.(2012)Sani, Lazaric, and Munos]{sani2013riskaversion}
Sani, A., Lazaric, A., and Munos, R.
\newblock Risk-aversion in multi-armed bandits.
\newblock In \emph{Advances in Neural Information Processing Systems}, pp.\
  3275--3283, 2012.

\bibitem[Soma \& Yoshida(2020)Soma and Yoshida]{soma2020statistical}
Soma, T. and Yoshida, Y.
\newblock Statistical learning with conditional value at risk.
\newblock \emph{arXiv preprint arXiv:2002.05826}, 2020.

\bibitem[Sun et~al.(2017)Sun, Dey, and Kapoor]{sun2016riskaware}
Sun, W., Dey, D., and Kapoor, A.
\newblock Risk-aversion in multi-armed bandits.
\newblock In \emph{International Conference on Machine Learning}, pp.\
  3280--3288, 2017.

\bibitem[Tamkin et~al.(2019)Tamkin, Keramati, Dann, and
  Brunskill]{tamkin2020dist}
Tamkin, A., Keramati, R., Dann, C., and Brunskill, E.
\newblock Distributionally-aware exploration for {CVaR} bandits.
\newblock In \emph{Neural Information Processing Systems 2019 Workshop on
  Safety and Robustness in Decision Making}, 2019.

\bibitem[Thompson(1933)]{thompson1933likelihood}
Thompson, W.~R.
\newblock On the likelihood that one unknown probability exceeds another in
  view of the evidence of two samples.
\newblock \emph{Biometrika}, 25\penalty0 (3/4):\penalty0 285--294, 1933.

\bibitem[Vakili \& Zhao(2016)Vakili and Zhao]{Vakili_2016}
Vakili, S. and Zhao, Q.
\newblock Risk-averse multi-armed bandit problems under mean-variance measure.
\newblock \emph{IEEE Journal of Selected Topics in Signal Processing},
  10\penalty0 (6):\penalty0 1093--1111, 2016.

\bibitem[Xi et~al.(2020)Xi, Tao, and Zhou]{xi2020nearoptimal}
Xi, G., Tao, C., and Zhou, Y.
\newblock Near-optimal {MNL} bandits under risk criteria.
\newblock \emph{arXiv preprint arXiv:2009.12511}, 2020.

\bibitem[Zhu \& Tan(2020)Zhu and Tan]{zhu2020thompson}
Zhu, Q. and Tan, V.~Y.
\newblock Thompson sampling algorithms for mean-variance bandits.
\newblock In \emph{International Conference on Machine Learning}, pp.\
  2645--2654, 2020.

\end{thebibliography}
